\documentclass{article}

\PassOptionsToPackage{authoryear}{natbib}
\usepackage[main,final]{neurips_2026}

\usepackage[utf8]{inputenc}
\usepackage[T1]{fontenc}
\usepackage{courier}
\usepackage{url}
\usepackage{booktabs}
\usepackage{amsfonts}
\usepackage{nicefrac}
\usepackage{microtype}
\usepackage{xcolor}

\usepackage{amsmath}
\usepackage{amssymb}
\usepackage{mathtools}
\usepackage{amsthm}

\usepackage{enumitem}

\usepackage{bm}
\newcommand\wbm{{\ensuremath{\bm{w}}}}

\newcommand\xbm{{\ensuremath{\bm{x}}}}
\newcommand\kbm{{\ensuremath{\bm{k}}}}
\newcommand\Kbm{{\ensuremath{\bm{K}}}}
\newcommand\fbm{{\ensuremath{\bm{f}}}}

\newcommand\ybm{{\ensuremath{\bm{y}}}}

\newcommand\Ibm{{\ensuremath{\bm{I}}}}

\newcommand{\Rbb}{\ensuremath{\mathbb{R} }}
\newcommand{\Pbb}{\ensuremath{\mathbb{P} }}
\newcommand{\Dbb}{\ensuremath{\mathbb{D} }}
\newcommand{\Ebb}{\ensuremath{\mathbb{E} }}

\newcommand{\Nbb}{\ensuremath{\mathbb{N} }}
\newcommand{\norm}[1]{\left\lVert {#1} \right\rVert}

\usepackage{algorithm}
\usepackage{algpseudocode}
\usepackage{float}
\usepackage{hyperref}
\hypersetup{
  hidelinks,
  pdftitle={Regret-Based (epsilon,delta)-optimal Stopping Criteria for Bayesian Optimization},
  pdfauthor={Haowei Wang; Jingyi Wang; Qiyu Wei}
}

\usepackage[capitalize,noabbrev]{cleveref}

\theoremstyle{plain}
\newtheorem{theorem}{Theorem}[section]

\newtheorem{lemma}[theorem]{Lemma}

\theoremstyle{definition}
\newtheorem{definition}[theorem]{Definition}
\newtheorem{assumption}[theorem]{Assumption}
\theoremstyle{remark}
\newtheorem{remark}[theorem]{Remark}

\usepackage[textsize=tiny]{todonotes}

\title{Regret-Based $(\epsilon,\delta)$-optimal Stopping Criteria for Bayesian Optimization}

\author{%
  Haowei Wang\thanks{Equal contribution; authors listed alphabetically.} \\
  National University of Singapore\\
  Singapore, Singapore\\
  haowei\_wang@u.nus.edu
  \And
  Jingyi Wang\footnotemark[1] \\
  Lawrence Livermore National Laboratory\\
  CA, USA\\
  wang125@llnl.gov
  \AND
  Qiyu Wei\footnotemark[1] \\
  The University of Manchester\\
  Manchester, UK\\
  qiyu.wei@postgrad.manchester.ac.uk
}

\begin{document}

\maketitle

\begin{abstract}
Bayesian optimization (BO) is a widely used iterative black-box optimization method that utilizes Gaussian process (GP) surrogate models. In practice, BO is typically terminated after a fixed evaluation budget is exhausted, which can incur unnecessary cost and provides no optimality guarantee on solution quality. Recent research in developing a practical stopping criterion has made empirical progress, yet a theoretically sound stopping criterion remains a work in progress. In this work, we present provably tighter instantaneous regret bounds for GP upper confidence bound (GP-UCB) at any given iteration. Then, we propose stopping criteria for GP-UCB based on this tighter bound that ensure an $\epsilon$-optimal solution with high probability $1-\delta$ upon termination. Numerical experiments are performed to validate and demonstrate the effectiveness and efficiency of our stopping criteria.
\end{abstract}

\section{Introduction}

Bayesian optimization (BO) has emerged as a powerful framework for the global optimization of expensive black-box functions.
In the classic form, it aims to solve the optimization problem
\begin{equation} \label{eqn:opt-prob}
 \centering
  \begin{aligned}
	  &\underset{\substack{\xbm}\in D}{\text{maximize}}
	  & & f(\xbm), \\
  \end{aligned}
\end{equation}
where $D\subset \Rbb^d$.
By combining a probabilistic surrogate model, most commonly a Gaussian process (GP), with an acquisition function that trades off exploration and exploitation to select observation points, BO selects evaluation points iteratively in order to locate the global optimum of the unknown objective function.
This iterative paradigm has enabled BO to achieve satisfactory performance in various applications where function evaluation is expensive or time-consuming, including engineering design~\citep{lam2018scaling,wang2023multifidelity}, simulation-based optimization~\citep{tay2022bayesian,wang2025am}, machine learning hyperparameter tuning~\citep{snoek2012practical}, and materials discovery~\citep{frazier2015bayesian,jin2023bayesian}. At iteration $t$, BO updates the GP posterior with past observations, maximizes the acquisition function to choose the next query, observes new data, and repeats. This sequential optimization loop structure immediately raises a fundamental algorithmic question: when should this iterative procedure stop?

Despite its centrality, the question of establishing principled termination criteria for BO has received limited theoretical attention. In practice, the prevailing approach is to prescribe a fixed budget in advance. This practice is ubiquitous across software packages and empirical studies~\citep{balandat2020botorch,lindauer2022smac3}, and is often chosen for convenience rather than algorithmic necessity. However, such total budget-based stopping rules are inherently ad hoc: they may lead to substantial inefficiencies, and more importantly they do not reflect the proximity to optimality, \textit{i.e.}, a sufficiently accurate global optimal solution is not guaranteed. This is especially problematic in settings where the underlying function evaluations arise from time-consuming simulations or costly physical experiments, for which exhausting a preallocated evaluation budget may be wasteful or impractical. This has motivated growing interest in designing BO stopping rules with optimality guarantees at termination~\citep{wilson2024stopping,xie2025cost}.

We seek stopping rules that are both theoretically grounded and practically effective. Concretely, a desirable stopping rule should i) guarantee convergence to the global optimizer under some standard regularity assumptions, and ii) yield substantial savings in evaluation cost relative to budget-exhaustion strategies. However, existing methods fall short of satisfying these goals. Existing stopping rules in BO often are either heuristic or restrictive. Heuristic approaches, such as using Monte Carlo estimates of simple regret~\citep{wilson2024stopping}, and its robustness improvement~\citep{he2025automatically}, confidence-bound gaps~\citep{makarova2022automatic}, or criteria based on Euclidean distance, probability of improvement, or EI values~\citep{lorenz2016stoppingcriteriaboostingautomatic,nguyen17a}, offer practical guidance but lack theoretical guarantees. Approaches with theoretical properties address a narrower stopping problem, add assumptions, or certify a weaker target: Bayesian optimal stopping~\citep{dai2019bayesian} terminates the training epochs of each evaluated configuration rather than the BO loop and assumes that performance improves monotonically over epochs; the gap of expected minimum simple regrets~\citep{ishibashi2023stopping} relies on additional stability assumptions and may stop while the simple regret remains large; and convexity-based termination~\citep{mcleod2018optimization,li2023not} identifies local convex structure and either switches to local optimization or certifies only local regret. Our criterion does not require these assumptions. In addition, some of these methods can incur considerable additional computational overhead. \citet{wilson2024stopping} proposed the Probabilistic Regret Bound (PRB) criterion, which estimates via Monte Carlo sampling from the GP posterior. It relies on numerical estimation subject to sampling and optimization errors. Building on PRB, \citet{he2025automatically} proposed MSP, which improves robustness to mean-model misspecification by sampling mean parameters from their posterior.

This gap motivates the present work.
To establish guaranteed optimality upon termination, the key idea is to better utilize the theoretical regret upper bounds. Since the seminal work of~\cite{srinivas2009gaussian},  there have been abundant studies on the regret bounds of GP-UCB, proving their sublinear cumulative regrets. However, the vast majority of the works have focused on the order of the regret upper bound in terms of iteration/sample number $t\in\Nbb$ as $t\to\infty$.
In contrast, this paper focuses on developing effective and efficient stopping criteria based on a {tight} instantaneous regret bound at a finite time $T$ that ensures an $\epsilon$ optimal solution with probability greater than $1-\delta$.

We evaluate the proposed criteria on both synthetic benchmarks and real-world optimization tasks. We compare the proposed stopping rules against widely used heuristics, uncertainty-based rules, and recent state-of-the-art (SOTA) baselines such as PRB~\citep{wilson2024stopping}.
Empirically, our criterion consistently reduces evaluation iterations while maintaining or improving the quality of the returned solutions. The observed stopping behavior closely aligns with the established theoretical findings and incurs incremental computational cost.
Our contributions can be summarized as follows.

\begin{itemize}
[nosep]
\item We propose theoretically grounded stopping criteria applicable to BO algorithms that admit subsequences of converging regret bounds including sublinear cumulative regret guarantees, e.g., GP-UCB.
\item We establish provably tighter instantaneous regret bounds for GP-UCB at given iteration $t$. Based on this new bound, we present explicit stopping criteria for GP-UCB.
\item We perform extensive experiments on both synthetic and real-world tasks. Results show that our criteria can save evaluations while achieving superior or comparable optimization performance to strong baselines.
\end{itemize}

\section{Background}\label{se:bo}

We consider the maximization problem~\eqref{eqn:opt-prob}.
At iteration $t$, we query $\xbm_t\in D$ and observe $y_t$.
BO algorithms consist of two major components: a GP surrogate model for the objective function $f$, and an acquisition function that measures the benefit of the next sample and guides the sample selection.
We introduce both in the subsequent subsections.

\subsection{Gaussian Processes}\label{se:gp}
A GP prior assumes that any finite collection of function values
follows a multivariate Gaussian distribution.
Consider a zero mean GP with the kernel $k(\xbm,\xbm'):\Rbb^d\times\Rbb^d\to\Rbb$, denoted as $GP(0,k(\xbm,\xbm'))$.
At each sample $\xbm_i\in D$, $i=1,2,\dots$, we consider observation noise $\epsilon_i$ for $f$ so that the observed function value is
$y_i = f(\xbm_i)+\epsilon_i$.
The noise is  assumed to follow an independent and identically distributed (i.i.d.) Gaussian distribution $\mathcal{N}(0,\sigma^2)$, where $\sigma>0$ is the noise standard deviation   (see Section~\ref{se:obj} for formal assumptions).

Denote the objective vector and the observation vector as  $f_{1:t}=[f(\xbm_1),\dots,f(\xbm_t)]^T$ and $\ybm_{1:t}=[y_1,\dots,y_t]^T$, respectively.
The posterior distribution of $f(\xbm) | \xbm_{1:t},\ybm_{1:t} \sim \mathcal{N} (\mu_t(\xbm),\sigma^2_t(\xbm))$, where $\mathcal{N}$ denotes the normal distribution, can be inferred using Bayes' rule with $\mu_t(\xbm) = \kbm_t(\xbm)^{\top} (\Kbm_t+ \sigma^2 \Ibm)^{-1} \ybm_{1:t}$ and $\sigma^2_t(\xbm) = k(\xbm,\xbm)-\kbm_t(\xbm)^{\top} (\Kbm_t+\sigma^2 \Ibm)^{-1}\kbm_t(\xbm)$,
where $\kbm_t(\xbm)= [k(\xbm_1,\xbm),\dots,k(\xbm_t,\xbm)]^T$ and $\Kbm_t = \bigl[k(\xbm_i,\xbm_j)\bigr]_{i,j=1}^t$.

Two of the most popular choices of kernels are the SE and Matérn kernels~\citep{frazier2018}, given by $k_{SE}(\xbm,\xbm') = \exp(-p^2/(2l^2))$ and $k_{Mat\acute{e}rn}(\xbm,\xbm') = \frac{1}{\Gamma(\nu)\,2^{\nu-1}} (\sqrt{2\nu}\,p/l)^{\nu} K_{\nu}(\sqrt{2\nu}\,p/l)$, where $l>0$ is a length scale hyperparameter, $p=\norm{\xbm-\xbm'}_2$, $\nu>0$ is the smoothness parameter of the Matérn kernel, and $K_{\nu}$ is the modified Bessel function of the second kind.

\subsection{Acquisition Function}\label{se:acquisition}
Denote the acquisition function given $t-1$ samples as $\alpha_{t-1}(\xbm)$.
The next sample $\xbm_{t}$ in a typical BO algorithm is chosen by maximizing the acquisition function over set $D_t$, \textit{i.e.},
$$\xbm_{t} = \arg\max_{\xbm\in D_t} \alpha_{t-1}(\xbm).$$
We use $D_t$ for its generality. For GP-UCB, $D_t$ is simply $D$ and thus is independent of $t$. For GP-TS however, $D_t$ can be a subset of $D$ and viewed as an algorithmic parameter. To solve this maximization, optimization algorithms such as L-BFGS can be used.

The choice of the acquisition function differentiates BO algorithms and has a significant impact on its performance. The definition of UCB for the maximization problem is $$UCB_{t}(\xbm) = \mu_{t-1}(\xbm) + \beta_t^{1/2}\sigma_{t-1}(\xbm),$$
where $\beta_t\geq 0$ is an algorithmic parameter crucial to the performance of GP-UCB. It is often chosen based on theoretical cumulative regret results~\citep{srinivas2009gaussian,chowdhury2017kernelized}. Naturally, $\beta_t$ plays an important role in the proposed stopping criteria for GP-UCB.

\subsection{Analytic Settings}\label{se:obj}
Let $\xbm^*\in\arg\max_{\xbm\in D} f(\xbm)$ and $f^*:=f(\xbm^*)$.
An output $\hat{\xbm}$ is said to be an $(\epsilon,\delta)$-optimal solution of the maximization problem if for given $\epsilon>0$ and $\delta\in(0,1)$, $\mathbb{P}(f^*-f(\hat{\xbm}) \le \epsilon)\ge 1-\delta$. Our stopping criteria in Section~\ref{se:convergence} are designed to guarantee this $(\epsilon,\delta)$-optimality upon termination.

We consider the optimization history next.
Let $\mathcal{H}_{t-1}=\{(\xbm_i,y_i)\}_{i=1}^{t-1}$ be the history and let  $\mathcal{F}_{t-1}:=\sigma(\mathcal{H}_{t-1})$ denote the filtration generated by the optimization history up to iteration $t-1$.
Given $\mathcal{H}_{t-1}$, the posterior mean $\mu_{t-1}$ and variance $\sigma_{t-1}^2$ are deterministic.
We use the standard notations $\mathcal{O}$ and $o$ to denote the orders of quantities. Further, we use $\tilde{\mathcal{O}}$ to suppress the logarithm terms of $\mathcal{O}$.

We present our analytic assumptions below.

\begin{assumption}\label{assp:bayesian}
The constraint set $D\subseteq [0,r]^d$ is convex and compact for $r>0$. The objective $f$ is a sample from the Gaussian process $GP(0,k(\xbm,\xbm'))$, where $k(\xbm,\xbm')\leq 1$ and $k(\xbm,\xbm)=1$.
There exist constants $a,b>0$ such that the objective function $f$ is Lipschitz continuous with Lipschitz constant $L$ (in 1-norm) with probability $\geq 1-d a e^{-L^2/b^2}$. The observation noise is i.i.d., \textit{i.e.}, $\epsilon_t\sim\mathcal{N}(0,\sigma^2),\sigma>0$ for all $t\in\Nbb$.
\end{assumption}

\section{Stopping Criteria}\label{se:convergence}
In this section, we propose our stopping criteria based on instantaneous regret bounds. Recall that the instantaneous regret is defined as $r_t = f^* - f(\xbm_t)$.
From an optimization perspective, it is natural to prefer BO algorithms with at least a converging subsequence of instantaneous regret, \textit{i.e.}, $r_{t_i} \to 0$ for $t_i\to\infty$.
Such a property would guarantee that an algorithm converges to the optimal solution $f^*$ at some iterates.
In the bandit setting, no-regret is a highly desirable property since it ensures overall convergence performance. A BO algorithm is called no-regret if its cumulative regret is sublinear, \textit{i.e.}, $\lim_{T\to\infty} R_T/T = 0$. Since $R_T$ is the sum of $r_t$, no regret guarantees a converging subsequence of instantaneous regret.

Consider a BO algorithm that satisfies the following instantaneous regret bound
\begin{equation} \label{eqn:instant-regret-bound}
 \centering
  \begin{aligned}
      \Pbb\{ r_t \leq  c_t(\xbm_{1:t-1},\ybm_{1:t-1},t,\delta), \text{for all} \ t\in\Nbb\}\geq 1-\delta.
  \end{aligned}
\end{equation}
The function $c_t$ is a $t$-dependent upper bound on $r_t$. By our formulation, $c_t$ can be computed using existing information accumulated up to iteration $t-1$, \textit{i.e.}, $\mathcal{F}_{t-1}$.
A well-known example of $c_t$ is $c_t(\xbm_{1:t-1},\ybm_{1:t-1},t,\delta)=2\beta_t^{1/2}\sigma_{t-1}(\xbm_t)$ for GP-UCB ~\citep{srinivas2009gaussian}.

Given a desired tolerance $\epsilon>0$, our stopping criterion terminates the algorithm if
\begin{equation} \label{def:stopping}
 \centering
  \begin{aligned}
      c_t(\xbm_{1:t-1},\ybm_{1:t-1},t,\delta) < \epsilon.
  \end{aligned}
\end{equation}

The stopping test~\eqref{def:stopping} uses the optimization trajectory and posterior uncertainty quantities available and introduces no additional black-box evaluations  beyond those at sample points $\xbm_i, i=1,\dots,t-1$. This makes it simple to implement and computationally efficient. We note that an upper bound $c_t$ that relies  on the observation $y_t$ can also be feasible.

With the proposed stopping criteria, the generic BO algorithm is given below in Algorithm~\ref{alg:bostop}.
\begin{algorithm}
 \caption{BO with stopping criteria}\label{alg:bostop}
  \begin{algorithmic}[1]
	  \State{Choose the kernel $k(\cdot,\cdot)$ and parameter $\sigma$.
      Choose $t_0$ initial samples $\xbm_i$ with  observations $y_i$,~$i=0,\dots,t_0-1$.}
	  \State{Train the GP surrogate model for $f$ on the initial samples.}
  \For{$t=t_0,t_0+1,\dots$}
	  \State{Find $\xbm_{t} = \arg\max_{\xbm\in D_t}\alpha_{t-1}(\xbm)$.}
	  \State{Observe $y_{t}=f(\xbm_{t})+\epsilon_{t}$. \;}
	  \State{Update the surrogate model with the addition of $\xbm_{t}$ and $y_{t}$.\;}
    \State{Compute the upper bound $c_t$.\;}
	  \If {\eqref{def:stopping}  is satisfied}
              \State{}{Exit}
	  \EndIf
  \EndFor
  \end{algorithmic}
\end{algorithm}

Clearly, if $c_t(\xbm_{1:t-1},\ybm_{1:t-1},t,\delta)$ in~\eqref{eqn:instant-regret-bound} converges, ~\eqref{def:stopping} provides the theoretical optimality guarantee, \textit{i.e.}, upon exit at iteration $t$, $\xbm_t$ is the $(\epsilon,\delta)$-optimal solution.
We formalize the optimality result below.
\begin{lemma}\label{lem:epsilonoptimal}
       If a BO algorithm generates a converging (to $0$) subsequence of the instantaneous regret bound $c_t$ as stated in~\eqref{eqn:instant-regret-bound}, \textit{i.e.}, $
\liminf_{t\to\infty} c_t = 0$, then for any $\epsilon>0$ and $\delta \in (0,1)$,
       Algorithm~\ref{alg:bostop} exits in a finite number of iterations $\bar{t}$. Further, $\xbm_{\bar{t}}$ is an $(\epsilon,\delta)$-optimal solution of~\eqref{eqn:opt-prob}.
\end{lemma}

    Lemma~\ref{lem:epsilonoptimal} asserts the well-posedness of the proposed stopping criteria in that~\eqref{def:stopping} is satisfied for large enough $t$.
    Since GP-UCB has been shown to be a no-regret algorithm in our setting (Assumption~\ref{assp:bayesian}) for commonly used kernels such as the SE and Matérn kernels~\citep{srinivas2009gaussian,whitehouse2023sublinear}, it satisfies  Lemma~\ref{lem:epsilonoptimal}. Hence, a $c_t$ that follows the actual regret $r_t$ closely can lead to a robust, efficient, and theoretically sound stopping criterion for GP-UCB.

\subsection{Tightened Instantaneous Regret Bounds for GP-UCB}\label{se:ucbts}
In this section, we present tightened instantaneous regret bounds for GP-UCB and the corresponding stopping criteria~\eqref{def:stopping}.
Importantly, our focus is on improving the finite-time upper bound at a given iteration $t$, rather than on the asymptotic order in $t$ as $t\to\infty$, since the tightness of the bound directly impacts the practical performance of the stopping criteria.

To ensure clear presentation and fair comparison, we keep the standard GP-UCB in Section~\ref{se:acquisition} and choose the widely adopted choice of $\beta_t$ in literature~\citep{srinivas2009gaussian}
\begin{equation} \label{eqn:ucb-bay-beta}
  \centering
  \begin{aligned}
   \beta_{t} =2\log(4\pi_t/\delta) +  2d\log(dt^2 br\sqrt{\log(4da/\delta)} ),
\end{aligned}
\end{equation}
  where $\pi_t=\pi^2t^2/6$.
  We note that the discretization size, \textit{i.e.}, number of discrete samples chosen for analytic purposes in~\eqref{eqn:ucb-bay-beta} is greater than or equal to $1$.
  Thus, implicitly, we have $dt^2 br\sqrt{\log(2da/\delta)}\geq 1$.
The instantaneous regret bound in~\cite{srinivas2009gaussian} is stated next.
\begin{lemma}\label{lemma:ucb-inst-regret-0}
    Pick  $\delta\in(0,1)$ and set $\beta_t$ as in~\eqref{eqn:ucb-bay-beta}.  Under Assumption~\ref{assp:bayesian}, the instantaneous regret of GP-UCB satisfies
    \begin{equation} \label{eqn:ucb-inst-regret-0}
  \centering
  \begin{aligned}
    r_t \leq 2 \beta_{t}^{1/2} \sigma_{t-1}(\xbm_t)+\frac{1}{t^2},
    \end{aligned}
\end{equation}
   for all $t\in\Nbb$, with probability $\geq 1-\delta$.
\end{lemma}
  To proceed, we define four analytic parameters
  \begin{equation} \label{eqn:opt-parameter}
  \centering
  \begin{aligned}
   n_{\eta,t},n_{\alpha,t} > 1,  s_t>0, n_L >1 .
  \end{aligned}
\end{equation}
  We emphasize that these four parameters are not deployed in the UCB acquisition function. The choices of their values will be given later in this section.
  The first two parameters $n_{\eta,t}$ and $n_{\alpha,t}$ define their corresponding confidence interval parameters $\eta_t$ and $\alpha_t$ via
     \begin{equation} \label{eqn:opt-parameter-ci}
  \centering
  \begin{aligned}
   \pi_t n_{\eta,t}(1-\Phi(\eta_t^{1/2}))= 1, \ \pi_t n_{\alpha,t}|\Dbb_t|(1-\Phi(\alpha_t^{1/2}))= 1,
  \end{aligned}
\end{equation}
  where $\Dbb_t$ is a time-dependent discretization defined below.
  The parameters $\eta_t$ and $\alpha_t$ serve similar purposes as $\beta_t$ and are used to obtain tighter confidence intervals in our analysis.

  The last two parameters $s_t$ and $n_L$ parameterize the construction of the discretization~\citep{srinivas2009gaussian}.
  Specifically, we can choose discretization $\Dbb_t$ of $D$ with  cardinality
   \begin{equation} \label{eqn:ucb-bay-discretization}
  \centering
  \begin{aligned}
  |\Dbb_t|
= \left\lceil(rdb)^d
\log\!\left(n_L d a /\delta\right)^{d/2}
t^{s_t d}\right\rceil,
 \end{aligned}
\end{equation}
where the ceiling function $\lceil \cdot\rceil$ denotes the smallest integer not less than its argument. Denote the closest point at $\xbm\in D$ to $\Dbb_t$ as $[\xbm]_t$. Then, we can bound $|f(\xbm)-f([\xbm]_t)|$ for $\forall \xbm\in D$ (Lemma~\ref{lemma:ucb-inst-regret-discretization}). We recover the widely used discretization with $\beta_t$ that ensures $|f(\xbm)-f([\xbm]_t)|\leq \frac{1}{t^2}$ if we set $s_t=2$ and $n_L=4$.

   Next, we present the instantaneous regret bound that includes these four parameters below.
 \begin{lemma}\label{lemma:ucb-inst-regret-new-form}
   Pick  $\delta\in(0,1)$, set $\beta_t$ as in~\eqref{eqn:ucb-bay-beta}, and let $\alpha_t\leq \beta_t$. Under Assumption~\ref{assp:bayesian}, at any given iteration $t$, the instantaneous regret bound of GP-UCB satisfies
    \begin{equation} \label{eqn:ucb-inst-regret-new-form}
  \centering
  \begin{aligned}
    r_t \leq (\beta_t^{1/2}+\eta_t^{1/2})\sigma_{t-1}(\xbm_t)-(\beta_t^{1/2}-\alpha_t^{1/2})\sigma \sqrt{\frac{1}{t+\sigma^2}}+\frac{1}{t^{s_t}},
    \end{aligned}
\end{equation}
   with probability $\geq 1-\delta/n_L- \frac{1}{\pi_t}(1/n_{\eta,t}+1/n_{\alpha,t})$, where $|\Dbb_t|$ satisfies \eqref{eqn:ucb-bay-discretization}, $\pi_t n_{\eta,t}(1-\Phi(\eta_t^{1/2}))= 1$, and $\pi_t n_{\alpha,t}|\Dbb_t|(1-\Phi(\alpha_t^{1/2}))= 1$.
\end{lemma}
The full proof of Lemma~\ref{lemma:ucb-inst-regret-new-form} is given in Appendix~\ref{se:ucb-ts-appx}.
 To establish~\eqref{eqn:ucb-inst-regret-new-form} at a given $t$ and $\delta$, we make multiple necessary modifications to the derivation of instantaneous regret upper bound.
  For a clear presentation, we list the key proof techniques below to highlight the changes.
\paragraph{Summary of Proof Techniques}

  1. Given $\mathcal{F}_{t-1}$, we use a different parameter $\eta_t$ for the confidence interval at $\xbm_t$, whose values can differ from $\beta_t$. Similarly, we introduce a new parameter $\alpha_t$ for the confidence interval at $[\xbm^*]_t$, independent of the value of $\beta_t$. In addition, we do not take the union bound over $\Dbb_t$ for the confidence interval at $\xbm_t$ with $\eta_t$, thereby reducing $\eta_t$ for the same probability level.   \\
  2.  We adopt a parameterized discretization $\Dbb_t$ with $s_t$ and $n_L$. This allows us to design and control the probability of the confidence interval at $[\xbm^*]_t$, since a union bound over all points in $\Dbb_t$ is needed.
  Moreover, $n_L$ allocates the probability budget of the Lipschitz continuity of $f$.\\
  3. Minor adjustment: we do not relax the probability of $r>c>0$ where $r\sim\mathcal{N}(0,1)$ from $1-\Phi(c)$ to $\frac{1}{2}e^{-c^2/2}$, as in~\cite{srinivas2009gaussian}. As $c$ increases, the ratio $\frac{1-\Phi(c)}{0.5e^{-c^2/2}}$ decreases. Moreover,  we only use one-sided confidence intervals to obtain an  upper bound on $r_t$, which increases the probability of the same confidence intervals holding.
   Finally, we incorporate the global lower bound of $\sigma_{t-1}(\xbm)$ for any $t\in \Nbb$ that was previously also used in~\cite{wang2014theoreybo}.
\paragraph{Sketch of Proof}
  From the definition of the instantaneous regret bound, we first decompose $r_t$ into $ f(\xbm^*)-f([\xbm^*]_t)+f([\xbm^*]_t)-f(\xbm_t)$, where $[\xbm]_t$ denotes the closest point in $\Dbb_t$ to $\xbm$. Then, we bound $f(\xbm^*)-f([\xbm^*]_t)$ through the Lipschitz continuity of $f$ with high probability and the discretization $\Dbb_t$.
  Next, we use the confidence interval at $[\xbm^*]_t$ and $\xbm_t$ with parameters  $\alpha_t^{1/2}$ (from $n_{\alpha,t}$) and $\eta_t^{1/2}$ (from $n_{\eta,t}$), respectively to replace $f([\xbm^*]_t)-f(\xbm_t)$ with terms $\mu_{t-1}([\xbm^*]_t)$, $\mu_{t-1}(\xbm_t)$, $\sigma_{t-1}([\xbm^*]_t)$, and $\sigma_{t-1}(\xbm_t)$. Given our condition $\beta_t\geq \alpha_t$, we use the global lower bound on $\sigma_{t-1}([\xbm^*]_t)$ in Lemma~\ref{lem:gp-sigma-bound} to arrive at \eqref{eqn:ucb-inst-regret-new-form}.

\paragraph{Optimization Subproblem Formulation}

The instantaneous regret bound in Lemma~\ref{lemma:ucb-inst-regret-new-form} depends on the parameters. To obtain the smallest upper bound, we solve an  optimization subproblem.
   We note that at iteration $t$ given $\mathcal{F}_{t-1}$, GP-UCB deterministically chooses $\xbm_t$. Thus,  $\sigma_{t-1}(\xbm_t)$ and similarly $\sigma \sqrt{\frac{1}{t+\sigma^2}}$ are deterministic given $\mathcal{F}_{t-1}$.
   Define $$c_1(t)=\sigma_{t-1}(\xbm_t) \quad  \text{and} \quad c_2(t)=\sigma \sqrt{\frac{1}{t+\sigma^2}}.$$
   For a given $\delta\in(0,1)$ and $t$, the optimization subproblem is
   \begin{equation} \label{eqn:ucb-inst-regret-opt}
  \centering
  \begin{aligned}
    \min_{n_{\eta,t}, n_{\alpha,t},s_t \in \Rbb}
    \quad &
    \eta_t^{1/2} c_1(t)
    + \alpha_t^{1/2} c_2(t)
    + \frac{1}{t^{s_t}}
    \\
    \text{s.t.}\quad
    &
    1/n_{\eta,t}+1/n_{\alpha,t}
    \leq \delta(1-1/n_L),
    \\
    &
    \pi_t n_{\eta,t}(1-\Phi(\eta_t^{1/2}))=1, \ \pi_t n_{\alpha,t}|\Dbb_t|(1-\Phi(\alpha_t^{1/2}))= 1,
    \\
    &
    |\Dbb_t| = (rdb)^d
    \log\!\left(n_L d a/\delta\right)^{d/2}
    t^{s_t d} +1,
    \\
    &
    \alpha_t,\eta_t,s_t >0,\\
    &
    \alpha_t^{1/2}\leq \beta_t^{1/2},\\
    &
    n_{\eta,t}, n_{\alpha,t}>1.
  \end{aligned}
\end{equation}

  The objective function in~\eqref{eqn:ucb-inst-regret-opt} is part of the upper bound in~\eqref{eqn:ucb-inst-regret-new-form} after dropping constants independent of the decision variables.
  The first constraint in~\eqref{eqn:ucb-inst-regret-opt} allocates the probability budget between the confidence intervals at $\xbm_t$ and $[\xbm^*]_t$. Combined with the weighting sequence $\pi_t$, this ensures that the instantaneous regret upper bound holds with probability at least $1-\delta$ simultaneously for all $t\in\Nbb$.
  The third constraint is from \eqref{eqn:ucb-bay-discretization} and the $+1$ term ensures the ceiling function is satisfied. We note that $|\Dbb_t|$ can be relaxed to be a continuous variable in the optimization subproblem.
    The constraint $\alpha_t\leq \beta_t$ is required to use the lower bound of $\sigma_{t-1}([\xbm^*]_t)$. Finally, the last line of constraints in~\eqref{eqn:ucb-inst-regret-opt} is explicitly written out for clarity of presentation, but is implicitly enforced by the probability equations.

Problem~\eqref{eqn:ucb-inst-regret-opt} is a simple three-dimensional ( $n_{\eta,t}, n_{\alpha,t}, s_t$) smooth, \textit{i.e.}, continuously differentiable, nonlinear optimization problem. Thus, it can be efficiently solved using standard optimization software such as IPOPT~\citep{ipopt} or SciPy's constrained optimization routines~\citep{2020SciPy-NMeth}, and should only add negligible computational cost. To ensure the robustness of the optimization solver, we can relax the lower box constraint on the optimization variables to
    \begin{equation} \label{eqn:ucb-inst-regret-opt-relaxed-bound}
  \centering
  \begin{aligned}
        \alpha_t^{1/2},\eta_t^{1/2},s_t \geq \epsilon_b,
        \qquad
         n_{\eta,t},n_{\alpha,t} \geq 1+\epsilon_b,
    \end{aligned}
\end{equation}
    where $\epsilon_b>0$ is a small lower-bound tolerance.

\begin{remark}[Well-posedness of the subproblem]
Using the relaxation in~\eqref{eqn:ucb-inst-regret-opt-relaxed-bound}, which is widely adopted in modern nonlinear optimization algorithms~\citep{ipopt}, the feasible region of~\eqref{eqn:ucb-inst-regret-opt} becomes closed. Moreover, the constraints imply that the feasible variables $n_{\eta,t}$, $n_{\alpha,t}$, and $s_t$ are bounded. Hence, every nonempty feasible level set of~\eqref{eqn:ucb-inst-regret-opt} is compact. We note that the existing bound~\eqref{eqn:ucb-inst-regret-0} is feasible for~\eqref{eqn:ucb-inst-regret-opt} when $n_L=4$ (given that $dt^2 br\sqrt{\log(2da/\delta)}\geq 1$ as mentioned above). Therefore, by the continuity of the objective function, a global minimizer of~\eqref{eqn:ucb-inst-regret-opt} exists~\citep{Nocedal_book}.
\end{remark}
\begin{remark}[Choice of $n_L$]
  We do not choose to include $n_L$ as an optimization variable for simplicity of the subproblem. Since $n_L$ controls the probability of the Lipschitz continuity of $f$, it does not require a union bound to hold true for all $t$. If we let $n_L$ be time-dependent, denoted by $n_{L,t}$, then the optimization subproblem needs to be modified accordingly and could include more constraints.  In practice, we recommend a large $n_{L}\geq 10$ that also ensures $n_L d a/\delta>1$ since its contribution  is inside a $\log$ function, whose effect on the bound is mild, for low-dimensional problems.
\end{remark}

Denote the solution to~\eqref{eqn:ucb-inst-regret-opt} as $\bar n_{\alpha,t},\bar n_{\eta,t}$, and $\bar s_t$. We can now state the tightened instantaneous regret bound. The proof is presented in Appendix~\ref{se:ucb-ts-appx}.
  \begin{theorem}\label{theorem:ucb-new-inst-regret}
       Under Assumption~\ref{assp:bayesian}, let $\delta\in(0,1)$ and define $\beta_t$ by~\eqref{eqn:ucb-bay-beta}. Then, with probability $\geq 1-\delta$, for all $ t\in\Nbb$, the instantaneous regret of GP-UCB satisfies
       \begin{equation} \label{eqn:ucb-inst-regret-new-form-solution}
  \centering
  \begin{aligned}
    r_t \leq (\beta_t^{1/2}+\bar \eta_t^{1/2})\sigma_{t-1}(\xbm_t)-( \beta_t^{1/2}-\bar \alpha_t^{1/2})\sigma \sqrt{\frac{1}{t+\sigma^2}}+\frac{1}{t^{\bar s_t}},
    \end{aligned}
\end{equation}
  where $\pi_t \bar n_{\eta,t}(1-\Phi(\bar \eta_t^{1/2}))= 1$, $\pi_t \bar n_{\alpha,t}|\Dbb_t|(1-\Phi(\bar \alpha_t^{1/2}))= 1$, and $\bar n_{\eta,t},\bar n_{\alpha,t},\bar s_t$ are the solutions to~\eqref{eqn:ucb-inst-regret-opt}.
\end{theorem}
For a given $t$ and $\delta$, the new bound is guaranteed to be no worse than~\eqref{eqn:ucb-inst-regret-0} since the existing bound~\eqref{eqn:ucb-inst-regret-0} is a feasible point of the optimization subproblem and can be recovered by setting $\alpha_t=\eta_t=\beta_t$, $s_t=2$ and $n_L=4$.
The stopping criterion \eqref{def:stopping} for GP-UCB is
\begin{equation} \label{eqn:stopping-ucb}
  \centering
  \begin{aligned}
     (\beta_t^{1/2}+\bar \eta_t^{1/2})\sigma_{t-1}(\xbm_t)-( \beta_t^{1/2}-\bar \alpha_t^{1/2})\sigma \sqrt{\frac{1}{t+\sigma^2}}+\frac{1}{t^{\bar s_t}} \leq \epsilon.
      \end{aligned}
\end{equation}

    We emphasize that evaluating the stopping criteria does not alter the GP-UCB
    optimization path prior to termination.  Practitioners may choose both the
    frequency and the iteration indices at which the stopping test
    \eqref{eqn:stopping-ucb} is performed (and hence when the optimization
    subproblem is solved). The computational cost induced by the stopping test is thus negligible.

\begin{remark}[Predetermined tests]
 If the stopping test is performed only at a finite set of predetermined
    iterations, the instantaneous regret bound can be further tightened, since
    a union bound over all $t\in\Nbb$ is no longer required. In particular,
    when the test is performed at a single fixed iteration, no union bound
    over $t$ is needed.
    Moreover, the improved bound can be evaluated post hoc at the end of an
    optimization run to obtain a tighter estimate of the instantaneous regret.
\end{remark}
\begin{remark}[Implications on the choice of $\beta_t$.]
Although we do not modify the GP-UCB
    algorithm or the commonly used choices of $\beta_t$, our analysis suggests
    that smaller values of $\beta_t$ may be admissible.
    In particular, proof techniques $2$ and $3$ can be incorporated into the design
    of $\beta_t$ to obtain tighter confidence intervals for a given finite $t$ (but not in order of $t$). Any modification to
    $\beta_t$ changes the exploration-exploitation trade-off in the GP-UCB
    acquisition function and therefore alters the algorithmic trajectory.
    However, since the order of the cumulative regret bound as $t\to\infty$
    is not improved, such modifications are unlikely to positively impact
    asymptotic performance. In contrast, for moderate values of $t$, the choice
    of $\beta_t$ can have a more pronounced effect on performance due to its
    influence on exploration. We leave a systematic investigation of this
    question for future work.
\end{remark}

\section{Experiments}\label{se:Experiments}
We first numerically illustrate the tightness of our new instantaneous regret bound, then evaluate the proposed regret-bound-based stopping criteria along two desiderata: validity (returning an $(\epsilon,\delta)$-optimal solution) and efficiency (saving evaluations). We additionally report the per-iteration computational overhead of each criterion.

\subsection{Effect of the New Instantaneous Regret Bound}
To quantify the improvement of our new bound~\eqref{eqn:ucb-inst-regret-new-form-solution} over the existing bound~\eqref{eqn:ucb-inst-regret-0}, we evaluate both at a given iteration $t$ with fixed problem constants $a=b=r=1$. Table~\ref{tab:ucb-bound-comparison-1} reports a representative configuration ($d=4$, $\delta=0.1$, $n_L=4$, $\sigma=10^{-2}$); our bound achieves over $30\%$ reduction (Ratio column) across three iteration counts $t\in\{20,100,500\}$ and three values of the posterior standard deviation $c_1(t)=\sigma_{t-1}(\xbm_t)$. Comparisons under additional values of $\delta$, $d$, and $n_L$ are provided in Appendix~\ref{se:ucb-ts-appx}.

\begin{table*}[t]
\centering
\caption{
Comparison between the existing bound~\eqref{eqn:ucb-inst-regret-0} and our new bound~\eqref{eqn:ucb-inst-regret-new-form-solution}.
Fixed parameters:
$\sigma=10^{-2}$,
$n_L=4$,
$a=b=r=1$,
$d=4$,
$\delta=0.1$,
and
$\pi_t=\pi^2 t^2/6$.
}
\label{tab:ucb-bound-comparison-1}
\setlength{\tabcolsep}{2pt}
\renewcommand{\arraystretch}{1.15}
\scriptsize
\resizebox{\linewidth}{!}
{
\begin{tabular}{ccccccccccccc}
\toprule
$t$
&
$c_1(t)$
&
$c_2(t)$
&
$\sqrt{\beta_t}$
&
$\sqrt{\eta_t}$
&
$\sqrt{\alpha_t}$
&
$\bar s_t$
&
$\bar n_{\eta,t}$
&
$\bar n_{\alpha,t}$
&
\eqref{eqn:ucb-inst-regret-new-form-solution}
&
\eqref{eqn:ucb-inst-regret-0}
&
Ratio
\\
\midrule
20
& $2.24\times10^{-3}$
& $2.24\times10^{-3}$
& $10.172$
& $3.775$
& $9.357$
& $2.325$
& $18.99$
& $44.78$
& $3.03\times10^{-2}$
& $4.55\times10^{-2}$
& $0.666$
\\
20
& $1.00\times10^{-2}$
& $2.24\times10^{-3}$
& $10.172$
& $3.708$
& $9.495$
& $2.329$
& $14.56$
& $158.30$
& $1.38\times10^{-1}$
& $2.03\times10^{-1}$
& $0.679$
\\
20
& $1.00\times10^{-1}$
& $2.24\times10^{-3}$
& $10.172$
& $3.688$
& $9.737$
& $2.338$
& $13.45$
& $1506.33$
& $1.39$
& $2.03$
& $0.681$
\\

100
& $1.00\times10^{-3}$
& $1.00\times10^{-3}$
& $11.647$
& $4.519$
& $10.035$
& $1.702$
& $19.55$
& $41.94$
& $1.49\times10^{-2}$
& $2.33\times10^{-2}$
& $0.642$
\\
100
& $1.00\times10^{-2}$
& $1.00\times10^{-3}$
& $11.647$
& $4.447$
& $10.238$
& $1.706$
& $13.93$
& $309.46$
& $1.60\times10^{-1}$
& $2.33\times10^{-1}$
& $0.687$
\\
100
& $1.00\times10^{-1}$
& $1.00\times10^{-3}$
& $11.647$
& $4.438$
& $10.465$
& $1.711$
& $13.39$
& $3044.43$
& $1.61$
& $2.33$
& $0.690$
\\

500
& $4.47\times10^{-4}$
& $4.47\times10^{-4}$
& $12.955$
& $5.162$
& $10.670$
& $1.400$
& $19.95$
& $40.19$
& $7.25\times10^{-3}$
& $1.16\times10^{-2}$
& $0.625$
\\
500
& $1.00\times10^{-2}$
& $4.47\times10^{-4}$
& $12.955$
& $5.091$
& $10.933$
& $1.404$
& $13.62$
& $636.44$
& $1.80\times10^{-1}$
& $2.59\times10^{-1}$
& $0.694$
\\
500
& $1.00\times10^{-1}$
& $4.47\times10^{-4}$
& $12.955$
& $5.087$
& $11.147$
& $1.407$
& $13.36$
& $6368.56$
& $1.80$
& $2.59$
& $0.696$
\\
\bottomrule
\end{tabular}}
\end{table*}

\subsection{Implementation Details}
We use a GP surrogate with Mat\'ern-$5/2$ and SE kernels, and use GP-UCB as the acquisition function for baselines with the same parameter $\beta_t$.
Each configuration is evaluated over $50$ independent random seeds.
Each run starts from $5$ initial points generated by a Sobol sequence, followed by sequential BO iterations up to the budget $T$ of the task.
To evaluate general performance, we did not perform post-hoc tuning of the GP parameters or $\epsilon$, and set $\epsilon=0.1$, $\delta=0.05$, and $a=b=1$ uniformly for all synthetic benchmark tasks. Here $a,b$ are the constants of Assumption~\ref{assp:bayesian} entering $\beta_t$; sensitivity to $a$, $b$, $\delta$, and $\epsilon$ is reported in Appendix~\ref{se:ablation-tables}. We report the mean stopping iteration and final simple regret over random seeds. We report the success rate as the fraction of runs whose final simple regret at stopping is $\le \epsilon$. Codes are available in \url{https://github.com/Qiyu-Wei/Regret_Based_Bo_Stopping.git}.

\subsection{Benchmarks}
We consider three families of problems.
(a) Synthetic functions:
We use classical continuous benchmarks with known optima: Branin ($d=2$, $T=128$), Rosenbrock ($d=4$, $T=96$), Levy ($d=4$, $T=96$), and Hartmann ($d=3$, $T=64$).
(b) GP-sampled objectives:
Following~\citet{wilson2024stopping}, we additionally evaluate on objectives drawn from GP priors with known optima. We consider two settings, $d=2$ with $T=128$ and $d=6$ with $T=512$, both at two observation noise levels $\sigma\in\{3\!\times\!10^{-3},\,5\!\times\!10^{-3}\}$. For each sampled function, the optimum is computed offline by the benchmark generation procedure and used as ground truth for regret evaluation. We report two variants: GP$^{\dagger}$ and GP, as in~\cite{wilson2024stopping}. In GP$^{\dagger}$, the BO surrogate is well-specified and uses the true hyperparameters of the GP data-generating process provided by an oracle. In GP, the surrogate hyperparameters are unknown and are estimated online via MAP.
(c) Real-world hyperparameter optimization:
We tune a random-forest classifier on the digits dataset (RFHPO, $d=5$, $T=128$) over the number of trees, the maximum depth, the minimum samples per split and per leaf, and the fraction of features, with a random $80/20$ train/validation split; and an RBF-kernel SVM on Fashion-MNIST (FMNIST-SVM, $d=2$, $T=64$) over $\log_{10}C\in[-2,3]$ and $\log_{10}\gamma\in[-4,0]$, with $2{,}000$ training and $2{,}000$ validation images. Following~\citet{wilson2024stopping}, we set $\epsilon=0.03$ and $0.165$ for these tasks. Consistent with the maximization framework~\eqref{eqn:opt-prob}, we maximize the negative validation error.

We compare against representative stopping criteria from the literature:
PRB \citep{wilson2024stopping},
$\Delta$CB \citep{makarova2022automatic} and $\Delta$ES \citep{ishibashi2023stopping},
and Acq \citep{wilson2024stopping}.
We also include two reference baselines: NOSTOP, which runs to the full budget $T$ without early stopping, and Oracle$_r$, which post-processes each run to determine the earliest iteration at which the regret satisfies the target tolerance; it stops at the first $t$ such that $f^* - \max_{i\le t} f(\xbm_i) \le \epsilon$, providing a lower bound on the stopping iteration achievable by any rule.
Our method is denoted UCB$_{\mathrm{br}}$, corresponding to the Bayesian stopping rule of Section~\ref{se:ucbts}. The subscript `b' indicates the Bayesian analytic setting and `r' stands for regret.

\subsection{Results}

\subsubsection{Full Results}

Table~\ref{tab:results} presents the experimental results under the Mat\'ern-$5/2$ kernel; results for the SE kernel are provided in Appendix~\ref{se:Additional Results}. We report the mean stopping iteration, success rate (fraction of runs achieving final simple regret $\le \epsilon$), and final simple regret for each method. The reported simple regret evaluates the incumbent $\xbm_t^+ = \arg\max_{i\le t} f(\xbm_i)$; since $f^* - f(\xbm_t^+) \le r_t$, satisfying this metric is at least as strong as the instantaneous-regret guarantee underlying our stopping criterion (Section~\ref{se:ucbts}).
UCB$_{\mathrm{br}}$ maintains the success rate of Oracle$_r$ on the GP$^{\dagger}$ problems and Rosenbrock, achieving an average success rate of $83.7\%$, which surpasses all competing methods except the oracle baseline.
Compared to PRB, UCB$_{\mathrm{br}}$ generally achieves higher success rates across problems while also attaining lower final regret. Among the criteria that stop before the budget, it attains the lowest final regret on $11$ of the $14$ problems. In contrast, Acq, $\Delta$CB, and $\Delta$ES exhibit inconsistent behavior, at times terminating too aggressively and at others being overly conservative, leading to unreliable performance across different problem settings. This demonstrates a favorable trade-off of our method between evaluation savings and solution quality.

UCB$_{\mathrm{br}}$ achieves competitive performance across all benchmarks in Table~\ref{tab:results}.
On the harder GP-learned $6$D variants where the surrogate hyperparameters are estimated online, UCB$_{\mathrm{br}}$'s $72\%$ and $60\%$ success contrasts sharply with PRB ($18$--$20\%$), $\Delta$CB ($18\%$), and $\Delta$ES ($12$--$14\%$), and exceeds that of Acq ($58$--$64\%$) with less than half of its evaluations.
On the synthetic benchmarks Branin and Rosenbrock, UCB$_{\mathrm{br}}$ remains competitive with success rates of $88\%$ and $100\%$, respectively, while reducing evaluations relative to NOSTOP.
On the two real-world HPO tasks, UCB$_{\mathrm{br}}$ achieves $78\%$ and $76\%$ success, $16$ points above the strongest baseline, while using about half of the evaluation budget.

\begin{table*}[t]
\centering
\caption{Comparison of stopping criteria under the Mat\'ern-$5/2$ kernel across GP-sampled functions ($^{\dagger}$ means known hyperparameters), synthetic benchmarks, and real-world tasks. All methods run GP-UCB with the same $\beta_t$, which UCB$_{\mathrm{br}}$ also uses in its certificate; each cell reports the mean stopping iteration (success rate / final simple regret) over $50$ seeds. For each problem, we mark in blue the practical method (UCB$_{\mathrm{br}}$, PRB, Acq, $\Delta$CB, $\Delta$ES) that stops in the fewest iterations among those reaching the target success rate $1-\delta = 95\%$; when no method reaches it, we mark the one with the highest success rate. NOSTOP and Oracle$_r$ are reference baselines and are excluded from the marking, as are methods that never stop before the budget $T$. For RFHPO and FMNIST-SVM, the final simple regret is the validation error, with $\epsilon=0.03$ and $\epsilon=0.165$, respectively.}
\label{tab:results}
\setlength{\tabcolsep}{2pt}
\renewcommand{\arraystretch}{1.15}
\scriptsize
\resizebox{\linewidth}{!}
{
\begin{tabular}{lcc*{7}{c}}
\toprule
Problem & $D$ & $T$ &
\textbf{NOSTOP} &
\textbf{Oracle$_r$} &
\textbf{UCB$_{\mathrm{br}}$} &
\textbf{PRB} &
\textbf{Acq} &
\textbf{$\Delta$CB} &
\textbf{$\Delta$ES} \\
\midrule
\textbf{GP\textsuperscript{$\dagger$}\,$3{\times}10^{-3}$} & 2 & 128 &
128.0 (100\% / 0.000347) &
6.0 (100\% / 0.0162) &
76.6 (100\% / 0.000347) &
6.4 (98\% / 0.0164) &
67.2 (100\% / 0.00794) &
9.6 (100\% / 0.00798) &
\textcolor{blue}{6.4 (98\% / 0.0183)} \\
\textbf{GP\textsuperscript{$\dagger$}\,$5{\times}10^{-3}$} & 2 & 128 &
128.0 (100\% / 0.000899) &
6.3 (100\% / 0.0181) &
115.2 (100\% / 0.000899) &
7.4 (94\% / 0.0219) &
69.5 (98\% / 0.0111) &
\textcolor{blue}{16.7 (100\% / 0.00351)} &
6.4 (94\% / 0.0267) \\
\textbf{GP\textsuperscript{$\dagger$}\,$3{\times}10^{-3}$} & 6 & 512 &
512.0 (100\% / 0.00202) &
9.6 (100\% / 0.0408) &
483.3 (100\% / 0.00202) &
\textcolor{blue}{13.9 (96\% / 0.0227)} &
330.1 (66\% / 0.0727) &
40.2 (100\% / 0.00691) &
7.4 (34\% / 0.148) \\
\textbf{GP\textsuperscript{$\dagger$}\,$5{\times}10^{-3}$} & 6 & 512 &
512.0 (100\% / 0.00297) &
11.7 (100\% / 0.0349) &
509.9 (100\% / 0.00297) &
16.2 (92\% / 0.0314) &
330.3 (70\% / 0.0663) &
\textcolor{blue}{64.5 (100\% / 0.00972)} &
6.9 (24\% / 0.172) \\
\textbf{GP\,$3{\times}10^{-3}$} & 2 & 128 &
128.0 (100\% / 0.000932) &
8.6 (100\% / 0.0209) &
\textcolor{blue}{29.4 (94\% / 0.0210)} &
6.2 (88\% / 0.0344) &
62.6 (94\% / 0.0206) &
6.3 (88\% / 0.0343) &
6.5 (92\% / 0.0302) \\
\textbf{GP\,$5{\times}10^{-3}$} & 2 & 128 &
128.0 (100\% / 0.00276) &
6.5 (100\% / 0.0212) &
\textcolor{blue}{38.2 (92\% / 0.0232)} &
6.0 (86\% / 0.0388) &
62.2 (92\% / 0.0229) &
6.4 (88\% / 0.0368) &
7.1 (90\% / 0.0341) \\
\textbf{GP\,$3{\times}10^{-3}$} & 6 & 512 &
512.0 (100\% / 0.00170) &
17.2 (100\% / 0.0410) &
\textcolor{blue}{123.9 (72\% / 0.0511)} &
8.0 (20\% / 0.174) &
300.6 (64\% / 0.0831) &
9.6 (18\% / 0.175) &
6.6 (12\% / 0.183) \\
\textbf{GP\,$5{\times}10^{-3}$} & 6 & 512 &
512.0 (100\% / 0.00306) &
21.6 (100\% / 0.0448) &
\textcolor{blue}{79.9 (60\% / 0.0865)} &
7.8 (18\% / 0.183) &
279.6 (58\% / 0.0958) &
9.5 (18\% / 0.181) &
6.5 (14\% / 0.190) \\
\textbf{Branin} & 2 & 128 &
128.0 (100\% / 0.00148) &
58.2 (100\% / 0.0491) &
\textcolor{blue}{87.0 (88\% / 0.313)} &
62.3 (82\% / 1.09) &
21.3 (18\% / 5.99) &
55.8 (66\% / 2.04) &
82.0 (62\% / 4.26) \\
\textbf{Rosenbrock} & 4 & 96 &
96.0 (100\% / 0.00608) &
8.5 (100\% / 0.0377) &
80.4 (100\% / 0.00868) &
\textcolor{blue}{61.1 (96\% / 0.0199)} &
75.8 (96\% / 0.0194) &
74.5 (96\% / 0.0195) &
86.2 (98\% / 0.0127) \\
\textbf{Levy} & 4 & 96 &
96.0 (98\% / 0.0448) &
56.0 (98\% / 0.0698) &
48.8 (42\% / 0.273) &
45.7 (42\% / 0.316) &
46.4 (42\% / 0.299) &
40.3 (36\% / 0.340) &
\textcolor{blue}{55.7 (56\% / 0.241)} \\
\textbf{Hartmann} & 3 & 64 &
64.0 (100\% / 0.0167) &
46.9 (100\% / 0.0464) &
53.2 (70\% / 0.388) &
34.8 (54\% / 1.09) &
64.0 (100\% / 0.0167) &
38.5 (52\% / 1.18) &
\textcolor{blue}{49.8 (74\% / 0.515)} \\
\textbf{RFHPO} & 5 & 128 &
128.0 (94\% / 0.0174) &
25.0 (94\% / 0.0250) &
\textcolor{blue}{57.3 (78\% / 0.0233)} &
15.0 (62\% / 0.0294) &
6.4 (30\% / 0.0376) &
16.1 (56\% / 0.0301) &
6.1 (30\% / 0.0383) \\
\textbf{FMNIST-SVM} & 2 & 64 &
64.0 (88\% / 0.155) &
16.9 (88\% / 0.160) &
\textcolor{blue}{32.6 (76\% / 0.158)} &
11.9 (60\% / 0.161) &
6.0 (48\% / 0.165) &
10.8 (60\% / 0.162) &
6.1 (48\% / 0.165) \\
\bottomrule
\end{tabular}}
\end{table*}

\begin{table*}[!t]
\centering
\small
\setlength{\tabcolsep}{4pt}
\caption{GPU runtime breakdown per BO iteration. Each stopping-rule cost is the mean per-iteration cost of evaluating the stopping test, with its share reported as the rule's cost divided by the sum of the rule's cost and the BO baseline.}
{
\begin{tabular}{lcccccccc}
\toprule
 & \multicolumn{4}{c}{BO baseline (ms)} & \multicolumn{4}{c}{Stopping rule cost time (ms) and share (\%)} \\
\cmidrule(lr){2-5}\cmidrule(lr){6-9}
Problem & GP fit & Acq opt & Obj eval & Total &
\textbf{UCB$_{\mathrm{br}}$} & \textbf{PRB} & \textbf{$\Delta$CB} & \textbf{Acq} \\
\midrule
\textbf{Branin}  &
196 & 76 & 0.06 & 272 &
\textbf{0.95 (0.36)} & 81.1 (23.7) & 110.7 (29.8) & 272.3 (51.0) \\
\textbf{Rosenbrock}  &
72 & 73 & 0.06 & 145 &
\textbf{0.92 (0.62)} & 80.5 (35.5) & 86.8 (37.3) & 146.2 (49.8) \\
\bottomrule
\end{tabular}
}
\label{tab:gpu_time_breakdown}
\end{table*}

\subsubsection{Test-time Overhead}
\label{sec:runtime_overhead}

We quantify the per-iteration wall-clock cost of different stopping rules to measure their practical overhead relative to the BO loop. All experiments are run on a GPU over two standard benchmarks (Branin-2D, Rosenbrock-4D). We consider four stopping methods: UCB$_{\mathrm{br}}$, PRB, $\Delta$CB, and Acq. For each problem, we run BO for 100 iterations and repeat across 50 random seeds. We decompose each iteration into GP fitting, acquisition optimization, and objective evaluation, and then measure the additional stopping-test time. Table~\ref{tab:gpu_time_breakdown} reports averaged per-iteration times. The reported stopping-test time is averaged over the actual stopping iterations.
The results show that UCB$_{\mathrm{br}}$ adds only $ 0.9$ ms per iteration ($<0.7\%$ of the full iteration time and $ 0.4\%$ relative to the BO baseline), making its overhead practically negligible. In contrast, PRB and $\Delta$CB incur substantial overhead (roughly 24--37\% of the full iteration time), and Acq can consume nearly half of the iteration time.

\section{Scope and Limitations}\label{se:limitations}
Our $(\epsilon,\delta)$ guarantee is established under Assumption~\ref{assp:bayesian}, the standard setting of the Bayesian GP-UCB analysis~\citep{srinivas2009gaussian}. Our experiments also cover settings beyond it, with hyperparameters estimated online and objectives that are not GP samples; under a misspecified kernel smoothness, the criterion is $\epsilon$-optimal in all $200$ runs (Appendix~\ref{se:ablation-nu}). The tightened bound is developed for GP-UCB, while the stopping principle of Lemma~\ref{lem:epsilonoptimal} applies to any BO algorithm with a computable, converging regret certificate, so tightened certificates for other acquisition functions can follow the same template. Finally, $\epsilon$ is task dependent; in practice, it is chosen as the smallest decision-relevant difference on the objective scale.

\section{Conclusions}\label{se:conclusion}
This paper proposes stopping criteria for BO algorithms based on regret bounds. The stopping criteria apply to BO algorithms with a converging subsequence of instantaneous regret bounds and ensure $(\epsilon,\delta)$-optimal solutions upon exit. We develop tight instantaneous regret bounds for GP-UCB with novel proof techniques. Using this new bound, we propose detailed forms of stopping criteria for GP-UCB.  Numerical experiments are conducted to demonstrate the performance of the proposed stopping criteria.

\section*{Acknowledgement}

This work was performed under the auspices of the U.S. Department of Energy by Lawrence Livermore
National Laboratory under Contract DE-AC52-07NA27344.

Funding was provided by the Office of Science, Office of Advanced Scientific Computing Research, Advancements in Artificial Intelligence for Science program.

\medskip
{
\small
\setcitestyle{numbers}
\bibliographystyle{plainnat}
\bibliography{neurips_2026}
}

\clearpage

\appendix

\section{Proofs of Section~\ref{se:convergence}}\label{se:proof-1-appx}
The proof of Lemma~\ref{lem:epsilonoptimal} is given below.
\begin{proof}
By the condition of the lemma, there exists a subsequence
$\{c_{t_i}\}$ such that
$c_{t_i}\to 0$ as $i\to\infty$.
Hence, for any $\epsilon>0$, there exists an index $\bar{i}$ such that
\[
c_{t_i}<\epsilon,
\qquad \forall i\ge \bar{i}.
\]
Therefore, the stopping criterion~\eqref{def:stopping} is satisfied at iteration
$\bar t:=t_{\bar i}$, and Algorithm~\ref{alg:bostop} exits in finite iterations.

From~\eqref{eqn:instant-regret-bound},
\[
r_{\bar t}=f^*-f(\xbm_{\bar t})
\le c_{\bar t}
\]
with probability $\geq 1-\delta$.
Since $c_{\bar t}<\epsilon$, we obtain
\[
r_{\bar t}\le \epsilon
\]
with probability $\geq 1-\delta$.
Hence, $\xbm_{\bar t}$ is an $(\epsilon,\delta)$-optimal solution.
\end{proof}
Next, we recall the definition of maximum information gain.
\begin{definition}\label{def:infogain}
For any set $A=\{\xbm_1,\dots,\xbm_T\}\subset D$, let $\ybm_A=[y(\xbm_1),\dots,y(\xbm_T)]^\top$ and $\fbm_A=[f(\xbm_1),\dots,f(\xbm_T)]^\top$. The maximum information gain after $T$ samples is $\gamma_T := \max_{A\subset D,\ |A|=T} I(\ybm_A;\fbm_A)$.
\end{definition}
Readers are referred to~\cite{cover1999elements,srinivas2009gaussian} for a detailed definition of the maximum information gain.
$\gamma_t$ is often used to quantify regret upper bounds~\citep{srinivas2009gaussian}, specifically the sum of posterior standard deviation at sample points. For the popular SE and Matérn kernels, the rates of $\gamma_t$ are well-established~\citep{iwazaki2025improved,vakili2021information}.

The following widely used lemma is useful in establishing cumulative regret bounds for BO algorithms
\begin{lemma}[\cite{srinivas2009gaussian}]\label{lem:variancebound}
The sum of GP posterior variances $\sigma_{i-1}^2(\xbm_{i})$ at the next samples $\xbm_{i}$ for $i=1,\dots,t$ satisfies
\begin{equation} \label{eqn:var-1}
  \centering
  \begin{aligned}
    \sum_{i=1}^t  \sigma_{i-1}^2(\xbm_i) \leq C_{\gamma} \gamma_t,
  \end{aligned}
\end{equation}
where $C_{\gamma} = \frac{2}{\log(1+\sigma^{-2})}$ and $\gamma_t$ is the maximum information gain after $t$ iterations or samples.
\end{lemma}

The following lemma provides the rate on $\gamma_t$ for SE and Matérn kernels.
\begin{lemma}[\cite{vakili2021information,iwazaki2025improved}]\label{lem:gammarate}
  For a GP with $t$ samples, the SE kernel has $\gamma_t=\mathcal{O}(\log^{d+1}(t))$, and
  the Matérn kernel with $\nu>\frac{1}{2}$ has $\gamma_t=\mathcal{O}(t^{\frac{d}{2\nu+d}}\log^{\frac{2\nu}{2\nu+d}} (t))$.
\end{lemma}

\section{Proofs and Additional Results for Section~\ref{se:ucbts}}\label{se:ucb-ts-appx}
We first recall the definition of the Reproducing kernel Hilbert space (RKHS).
\begin{definition}\label{def:rkhs}
    Let $k$ be a positive definite kernel $k: \mathcal{X}\times \mathcal{X}\to\Rbb $ with respect to a finite Borel measure supported on $\mathcal{X}$. A Hilbert space $H_k$ of functions on $\mathcal{X}$ with inner product $\langle \cdot,\cdot \rangle_{H_k}$ is called a RKHS with kernel $k$ if $k(\cdot,\xbm)\in H_k$ for all $\xbm\in \mathcal{X}$, and $\langle f,k(\cdot,\xbm)\rangle_{H_k}=f(\xbm)$ for all $\xbm\in \mathcal{X}, f\in H_k$. The induced RKHS norm $\norm{f}_{H_k}=\sqrt{\langle f,f\rangle_{H_k}}$ measures the smoothness of $f$ with respect to the kernel $k$.
 \end{definition}
 Next, the noisy frequentist assumptions are given.
 \begin{assumption}\label{assp:freq}
    The frequentist setting satisfies the following conditions.
 	The function $f$ lies in the RKHS, denoted as $\mathcal{H}_k(D)$ associated with the bounded kernel $k$, with the norm $\norm{\cdot}_{H_k}$.
          The kernels satisfies $k(\xbm,\xbm')\leq 1$, and  $k(\xbm,\xbm)= 1$ for $\forall \xbm,\xbm'\in D$.
 	The RKHS norm is bounded above by a constant  $B\geq 0$ , \textit{i.e.}, $\norm{f}_{H_k} \leq B$.
     Moreover, the domain $D$ is compact. The noise $\epsilon_t$ is conditionally $R$-sub-Gaussian for a fixed $R>0$, \textit{i.e.},
     \begin{equation} \label{eqn:sub-gaussian}
  \centering
   \begin{aligned}
        \forall t\geq 0, \forall \lambda\in\Rbb, \Ebb[e^{\lambda \epsilon_t}|\mathcal{F}_{t-1}]\leq \exp\left(\frac{\lambda^2R^2}{2}\right).
   \end{aligned}
 \end{equation}
 \end{assumption}

We first recall important lemmas for confidence intervals under Assumption~\ref{assp:bayesian}.
\begin{lemma}\label{lem:phi}
The PDF and CDF of the standard normal distribution satisfy $0< \phi(x)\leq \phi(0), \Phi(x)\in(0,1)$,
for any $x\in\Rbb$.
 Given a random variable $r$ that follows the standard normal distribution: $r\sim\mathcal{N}(0,1)$, the probability of $r>c, c>0$, satisfies $\Pbb\{r>c\}= 1-\Phi(c)$.
 Similarly, for $c<0$, $\Pbb\{r<c\}= \Phi(c)$.
\end{lemma}

The following lemma recalls a global lower bound of $\sigma_{t-1}(\xbm)$.
\begin{lemma}\label{lem:gp-sigma-bound}
    For $\forall \xbm\in D$ and $t\in\Nbb$,  the posterior standard deviation has the lower bound
 \begin{equation} \label{eqn:gp-sigma-bound-1}
 \centering
  \begin{aligned}
     \sigma_t(\xbm) \geq  \sigma \sqrt{\frac{1}{t+\sigma^2}}.
   \end{aligned}
\end{equation}
\end{lemma}
Lemma~\ref{lem:gp-sigma-bound} is a widely accepted result and can be found in~\cite{wang2014theoreybo}.

We now recall  (see Section 5 in~\cite{srinivas2009gaussian}) and introduce changes to the confidence intervals.
\begin{lemma}\label{lem:fmu}
For any $t\in \Nbb $, $\xbm\in D$ deterministic with $\mathcal{F}_{t-1}$, and $\beta>0$,
\begin{equation} \label{eqn:fmu-1}
 \centering
  \begin{aligned}
     |f(\xbm) - \mu_{t-1}(\xbm)  | \leq \beta^{1/2} \sigma_{t-1}(\xbm),
  \end{aligned}
\end{equation}
holds with probability $ 1-2(1-\Phi(\beta^{1/2}) )$.
Moreover, the one-sided inequalities hold with probability $ \Phi(\beta^{1/2}) $, \textit{i.e.},    with probability $ \Phi(\beta^{1/2}) $,
     $f(\xbm) - \mu_{t-1}(\xbm)   \leq \beta^{1/2} \sigma_{t-1}(\xbm)$,
and with probability $\geq \Phi(\beta^{1/2})$,
     $f(\xbm) - \mu_{t-1}(\xbm)   \geq -\beta^{1/2} \sigma_{t-1}(\xbm)$.
\end{lemma}
\begin{proof}
 Under Assumption~\ref{assp:bayesian}, given $\mathcal{F}_{t-1}$, $f(\xbm)\sim \mathcal{N}(\mu_{t-1}(\xbm),\sigma_{t-1}^2(\xbm))$. Then,
 \begin{equation} \label{eqn:fmu-a-pf-1}
 \centering
  \begin{aligned}
     \Pbb\left\{ f(\xbm)-\mu_{t-1}(\xbm) > \beta^{1/2} \sigma_{t-1}(\xbm)\right\} = 1-\Phi(\beta^{1/2} ) .
  \end{aligned}
 \end{equation}
  Similarly,
 \begin{equation} \label{eqn:fmu-a-pf-2}
 \centering
  \begin{aligned}
     \Pbb\left\{ f(\xbm)-\mu_{t-1}(\xbm) < -\beta^{1/2} \sigma_{t-1}(\xbm)\right\} = 1-\Phi(\beta^{1/2}).
  \end{aligned}
 \end{equation}
  Thus, by union bound,
   \begin{equation*} \label{eqn:fmu-a-pf-3}
 \centering
  \begin{aligned}
     \Pbb\left\{ |f(\xbm)-\mu_{t-1}(\xbm)| \leq \beta^{1/2} \sigma_{t-1}(\xbm)\right\} = 1- 2( 1- \Phi(\beta^{1/2})).
  \end{aligned}
 \end{equation*}
\end{proof}

Lemma~\ref{lem:fmu} is extended to a discrete set $\Dbb_t$ and all $t\in\Nbb$ next.
\begin{lemma}\label{lem:discrete-fmu}
  For any given $t\in \Nbb $, a discrete set $\Dbb_t \subseteq D$, and $\alpha_t>0$, we have
\begin{equation} \label{eqn:discrete-fmu-1}
 \centering
  \begin{aligned}
     f(\xbm) - \mu_{t-1}(\xbm)   \leq \alpha_t^{1/2} \sigma_{t-1}(\xbm),
  \end{aligned}
\end{equation}
holds with probability $\geq 1-|\Dbb_t| (1-\Phi(\alpha_t^{1/2}))$, simultaneously for all $\xbm\in \Dbb_t$.
\end{lemma}
The proof is the same as that for Lemma 5.6 in~\cite{srinivas2009gaussian}.
Given that $\xbm^*$ is random with respect to $\mathcal{F}_{t-1}$, we take the union bound over $\Dbb_t$, as in~\cite{srinivas2009gaussian}.
 Further,~\eqref{eqn:discrete-fmu-1} is only one-sided.

We now consider the confidence interval at $\xbm_t$.
    \begin{lemma}\label{lem:fmu-t}
For any given $t\in \Nbb $ and parameter $\eta_t>0$,
\begin{equation} \label{eqn:fmu-t-1}
 \centering
  \begin{aligned}
     f(\xbm_{t}) - \mu_{t-1}(\xbm_{t})  \geq -\eta_t^{1/2} \sigma_{t-1}(\xbm_{t}),
  \end{aligned}
\end{equation}
holds with probability $\geq 1-(1-\Phi(\eta_t^{1/2}))$.
\end{lemma}
Lemma~\ref{lem:fmu-t} is a direct result of Lemma~\ref{lem:fmu} since $\xbm_t$ is determined with $\mathcal{F}_{t-1}$.

Next, we discuss our parameterized discretization $\Dbb_t$ of $D$.
For any $t\in\Nbb$, $\Dbb_t$ is a predetermined set of finite points in $D$.
The discretization $\Dbb_t$ aims to cover the compact set $D$ with just enough points so that for any $\xbm\in D$ and $t\in\Nbb$, the distance between $\xbm\in D$ and $\Dbb_t$ is small enough and controlled.  As earlier, denote the closest point of $\xbm$ to $\Dbb_t$ as
\begin{equation} \label{eqn:close-point}
 \centering
  \begin{aligned}
 	[\xbm]_t := \underset{\substack{\wbm}\in \Dbb_t}{\text{argmin}}
	   \norm{\xbm-\wbm} ,
  \end{aligned}
\end{equation}
and define the parameter $h_t$ as a measure of the distance between $D$ and $\Dbb_t$ by
\begin{equation} \label{eqn:compactdisc-1}
 \centering
  \begin{aligned}
     \norm{\xbm-[\xbm]_t}_1 \leq h_t, \ \forall \xbm\in D,
  \end{aligned}
\end{equation}
  where $\norm{\cdot}_1$ is the one-norm.
In this paper, we adopt the standard discretizations with decreasing distances, similar to that in~\cite{srinivas2009gaussian}.
By~\eqref{eqn:compactdisc-1}, $\Dbb_t$ that corresponds to $h_t$ is constructed with $(\frac{rd}{h_t})^d$ points.
Given Assumption~\ref{assp:bayesian} and $D \subseteq [0,r]^d$, this leads to a choice of
   \begin{equation} \label{eqn:disc-size-1}
  \centering
  \begin{aligned}
     |\Dbb_t| = \left\lceil\left( \frac{rd}{h_t}\right)^d\right\rceil,
   \end{aligned}
  \end{equation}
where the points in $\Dbb_t$ are evenly spaced.
We emphasize that the discretization does not play a role in the GP-UCB algorithm. Rather, it is an analytic tool to help analysis.
We choose
$$h_t = \frac{1}{ b\sqrt{\log( n_L da/\delta)} t^{s_t}},$$
where $s_t>0$ is a  parameter in~\eqref{eqn:opt-parameter} and $n_L\geq 1$  is the probability associated with Lipschitz continuity.  It is easy to verify that the choice~\eqref{eqn:ucb-bay-discretization} satisfies~\eqref{eqn:disc-size-1}.
We have the following Lemma on the Lipschitz continuity of $f$ (see also Lemma 5.7 in~\cite{srinivas2009gaussian}).
\begin{lemma}\label{lemma:ucb-inst-regret-discretization}
Let the discretization be such that~\eqref{eqn:ucb-bay-discretization} is satisfied. Then, under Assumption~\ref{assp:bayesian}, for all $\xbm\in D$ and $t\in\Nbb$,
    \begin{equation} \label{eqn:ucb-inst-regret-lip}
  \centering
  \begin{aligned}
    |f(\xbm) - f([\xbm]_t)| \leq \frac{1}{t^{s_t}},
    \end{aligned}
\end{equation}
   with probability $\geq 1-\delta/n_L$.
\end{lemma}
We emphasize that the discretization $\Dbb_t$ in~\cite{srinivas2009gaussian} is recovered if $s_t=2$ and $n_L=4$.

Next, we consider a one-sided confidence interval at  $[\xbm^*]_t$ using Lemma~\ref{lem:discrete-fmu}.

\begin{lemma}\label{lem:fmu-ct}
  For any given $t\in \Nbb $, the discretization $\Dbb_t \subseteq D$, and $\alpha_t>0$, we have
\begin{equation} \label{eqn:fmu-ct-1}
 \centering
  \begin{aligned}
     f([\xbm^*]_{t}) - \mu_{t-1}([\xbm^*]_{t})   \leq \alpha_t^{1/2} \sigma_{t-1}([\xbm^*]_{t}),
  \end{aligned}
\end{equation}
holds with probability $\geq 1-|\Dbb_t| (1-\Phi(\alpha_t^{1/2}))$.
\end{lemma}
  From the definitions of the parameters \eqref{eqn:opt-parameter} and \eqref{eqn:opt-parameter-ci}, we have
  \begin{equation} \label{eqn:eta-probability}
 \centering
  \begin{aligned}
   \pi_t n_{\eta,t}(1-\Phi(\eta_t^{1/2}))= 1
   \end{aligned}
\end{equation}
   and
   $$ \pi_t n_{\alpha,t}|\Dbb_t|(1-\Phi(\alpha_t^{1/2}))= 1,$$

  The proof of Lemma~\ref{lemma:ucb-inst-regret-new-form} is given below.
 \begin{proof}
     At a given $t\in\Nbb$, we start from the definition of the instantaneous regret bound $$r_t = f(\xbm^*)-f(\xbm_t).$$
  From Lemma~\ref{lem:fmu-t} and~\eqref{eqn:eta-probability}, we have the confidence interval
  \begin{equation} \label{eqn:ucb-inst-regret-change-1}
  \centering
  \begin{aligned}
    f(\xbm_t) - \mu_{t-1}(\xbm_t) \geq -\eta_{t}^{1/2} \sigma_{t-1}(\xbm_t),
    \end{aligned}
\end{equation}
    with probability $\geq 1-\frac{1}{\pi_t n_{\eta,t}}$, where $\eta_{t}^{1/2}>0$ is a parameter determined by $n_{\eta,t}$ and can be different from $\beta_t^{1/2}$.

    From the discretization and Lemma~\ref{lemma:ucb-inst-regret-discretization}, we have \eqref{eqn:ucb-inst-regret-lip} for $\forall \xbm\in D$ with probability $\geq 1-\delta/n_L$.
    Applying~\eqref{eqn:ucb-inst-regret-change-1}  and~\eqref{eqn:ucb-inst-regret-lip} at $\xbm^*$, we can write the instantaneous regret bound as
\begin{equation} \label{eqn:ucb-inst-regret-2}
  \centering
  \begin{aligned}
    r_t \leq f([\xbm^*]_t)-\mu_{t-1}(\xbm_t) +  \eta_{t}^{1/2} \sigma_{t-1}(\xbm_t)+\frac{1}{t^{s_t}},
    \end{aligned}
\end{equation}
   with probability $\geq 1-\frac{1}{\pi_t n_{\eta,t}} - \delta/n_L$.
   Next, we replace $f([\xbm^*]_t)$ in~\eqref{eqn:ucb-inst-regret-2} via Lemma~\ref{lem:discrete-fmu} at $[\xbm^*]_t$ to obtain
   \begin{equation} \label{eqn:ucb-inst-regret-3}
  \centering
  \begin{aligned}
    r_t \leq \mu_{t-1}([\xbm^*]_t)-\mu_{t-1}(\xbm_t) +\alpha_t^{1/2}\sigma_{t-1}([\xbm^*]_t)+  \eta_{t}^{1/2} \sigma_{t-1}(\xbm_t)+\frac{1}{t^{s_t}},
    \end{aligned}
\end{equation}
  with probability $\geq 1-\frac{1}{\pi_t n_{\eta,t}}-\frac{1}{\pi_t n_{\alpha,t}} - \delta/n_L$. From the definition of UCB, we have
    \begin{equation} \label{eqn:ucb-inst-regret-new-form-pf-1}
  \centering
  \begin{aligned}
    \mu_{t-1}([\xbm^*]_t)+\beta_t^{1/2}\sigma_{t-1}([\xbm^*]_t) \leq \mu_{t-1}(\xbm_t) +\beta_t^{1/2}\sigma_{t-1}(\xbm_t).
    \end{aligned}
\end{equation}

   Using~\eqref{eqn:ucb-inst-regret-new-form-pf-1} in~\eqref{eqn:ucb-inst-regret-3}, we have
\begin{equation} \label{eqn:ucb-inst-regret-new-form-pf-2}
  \centering
  \begin{aligned}
    r_t \leq& \beta_t^{1/2}\sigma_{t-1}(\xbm_t)-\beta_t^{1/2}\sigma_{t-1}([\xbm^*]_t)+\alpha_t^{1/2}\sigma_{t-1}([\xbm^*]_t)+  \eta_{t}^{1/2} \sigma_{t-1}(\xbm_t)+\frac{1}{t^{s_t}}\\
        \leq& \beta_t^{1/2}\sigma_{t-1}(\xbm_t)-(\beta_t^{1/2}-\alpha_t^{1/2})\sigma_{t-1}([\xbm^*]_t)+ \eta_{t}^{1/2} \sigma_{t-1}(\xbm_t)+\frac{1}{t^{s_t}},
    \end{aligned}
\end{equation}
with probability $\geq 1-\frac{1}{\pi_t n_{\eta,t}}-\frac{1}{\pi_t n_{\alpha,t}} - \delta/n_L$.
 Since $\alpha_t\leq \beta_t$, by the global lower bound of $\sigma_{t-1}(\xbm)$ in Lemma~\ref{lem:gp-sigma-bound}, \eqref{eqn:ucb-inst-regret-new-form-pf-2} implies
    \begin{equation} \label{eqn:ucb-inst-regret-new-form-pf-3}
  \centering
  \begin{aligned}
    r_t \leq \beta_t^{1/2}\sigma_{t-1}(\xbm_t)-(\beta_t^{1/2}-\alpha_t^{1/2})\sigma \sqrt{\frac{1}{t+\sigma^2}}+  \eta_{t}^{1/2} \sigma_{t-1}(\xbm_t)+\frac{1}{t^{s_t}},
    \end{aligned}
\end{equation}
with probability $\geq 1-\frac{1}{\pi_t n_{\eta,t}}-\frac{1}{\pi_t n_{\alpha,t}} - \delta/n_L$, which leads to \eqref{eqn:ucb-inst-regret-new-form}.

 \end{proof}
 \begin{remark}
Without the assumption $\alpha_t\leq \beta_t$, the coefficient of
$\sigma_{t-1}([\xbm^*]_t)$ in~\eqref{eqn:ucb-inst-regret-new-form-pf-2}
may become positive, in which case the global lower bound on
$\sigma_{t-1}([\xbm^*]_t)$ is no longer useful. From the first inequality
in~\eqref{eqn:ucb-inst-regret-new-form-pf-2}, decreasing $\alpha_t^{1/2}$
tightens the upper bound on $r_t$. Thus, we impose
$\alpha_t^{1/2}\leq \beta_t^{1/2}$, which makes the coefficient nonpositive
and allows us to use the global lower bound for
$\sigma_{t-1}([\xbm^*]_t)$.
\end{remark}

Proof of Theorem~\ref{theorem:ucb-new-inst-regret} is given next.
\begin{proof}
For each  $t$, Lemma~\ref{lemma:ucb-inst-regret-new-form} applies with
$\bar n_{\eta,t}$, $\bar n_{\alpha,t}$, and $\bar s_t$ obtained from
\eqref{eqn:ucb-inst-regret-opt}.
By the first constraint in~\eqref{eqn:ucb-inst-regret-opt},
\begin{equation} \label{eqn:ucb-new-inst-regret-pf-1}
  \centering
  \begin{aligned}
    \frac{1}{\bar n_{\eta,t}}+\frac{1}{\bar n_{\alpha,t}}
    &\le \delta\left(1-\frac1{n_L}\right).
  \end{aligned}
\end{equation}
Therefore, at a given $t$, we have
\begin{equation} \label{eqn:ucb-new-inst-regret-pf-2}
  \centering
  \begin{aligned}
    \mathbb P\left\{\eqref{eqn:ucb-inst-regret-new-form}\ \text{holds at }t\in\Nbb\right\}
    &\ge 1-\frac{1}{\pi_t}
    \left(\frac{1}{\bar n_{\eta,t}}+\frac{1}{\bar n_{\alpha,t}}\right)
    -\frac{\delta}{n_L}\\
    &\geq 1-\frac{1}{\pi_t}\delta\left(1-\frac1{n_L}\right)-\frac{\delta}{n_L}.
  \end{aligned}
\end{equation}
The Lipschitz/discretization inequality in
Lemma~\ref{lemma:ucb-inst-regret-discretization} holds simultaneously for all
$t\in\Nbb$ with probability  $\geq 1-\delta/n_L$.
Applying union
bound  to the $t$-dependent confidence intervals leads to
\begin{equation} \label{eqn:ucb-new-inst-regret-pf-3}
  \centering
  \begin{aligned}
    \mathbb P\left\{\eqref{eqn:ucb-inst-regret-new-form}\ \text{holds for all }t\in\Nbb\right\}
    &\ge 1-\sum_{t=1}^{\infty}\frac{1}{\pi_t}\delta\left(1-\frac1{n_L}\right)
    -\frac{\delta}{n_L}\\
    &=1-\delta,
  \end{aligned}
\end{equation}
where the equality uses $\sum_{t=1}^{\infty}1/\pi_t=1$.
\end{proof}
\begin{remark}
    We note that the probability~\eqref{eqn:ucb-new-inst-regret-pf-2} can be derived first by conditioning on the history $\mathcal{H}_{t-1} $ and filtration $\mathcal{F}_{t-1}$. Then, the conditional probability lead to probability by taking the expectation of both sides. Such technique is standard in BO literature~\citep{srinivas2009gaussian}.
\end{remark}

More comparison between the existing bound~\eqref{eqn:ucb-inst-regret-0} and our new bound~\eqref{eqn:ucb-inst-regret-new-form-solution}  with different constants are given below, where the latter consistently reduces the upper bound.
\begin{table*}[t]
\centering
\caption{
Comparison between the existing bound~\eqref{eqn:ucb-inst-regret-0} and our new bound~\eqref{eqn:ucb-inst-regret-new-form-solution}.
Fixed parameters:
$\sigma=10^{-2}$,
$n_L=10$,
$a=b=r=1$,
$d=2$,
$\delta=0.1$,
and
$\pi_t=\pi^2 t^2/6$.
}
\label{tab:ucb-bound-comparison-2}
\setlength{\tabcolsep}{2pt}
\renewcommand{\arraystretch}{1.15}
\scriptsize
\resizebox{\linewidth}{!}
{
\begin{tabular}{ccccccccccccc}
\toprule
$t$
&
$c_1(t)$
&
$c_2(t)$
&
$\sqrt{\beta_t}$
&
$\sqrt{\eta_t}$
&
$\sqrt{\alpha_t}$
&
$\bar s_t$
&
$\bar n_{\eta,t}$
&
$\bar n_{\alpha,t}$
&
\eqref{eqn:ucb-inst-regret-new-form-solution}
&
\eqref{eqn:ucb-inst-regret-0}
&
Ratio
\\
\midrule
20
& $2.24\times10^{-3}$
& $2.24\times10^{-3}$
& $7.468$
& $3.751$
& $7.050$
& $2.464$
& $17.28$
& $31.13$
& $2.48\times10^{-2}$
& $3.34\times10^{-2}$
& $0.742$
\\
20
& $1.00\times10^{-2}$
& $2.24\times10^{-3}$
& $7.468$
& $3.668$
& $7.224$
& $2.472$
& $12.43$
& $104.61$
& $1.11\times10^{-1}$
& $1.49\times10^{-1}$
& $0.746$
\\
20
& $1.00\times10^{-1}$
& $2.24\times10^{-3}$
& $7.468$
& $3.643$
& $7.468$
& $2.430$
& $11.26$
& $840.33$
& $1.11$
& $1.49$
& $0.744$
\\

100
& $1.00\times10^{-3}$
& $1.00\times10^{-3}$
& $8.665$
& $4.499$
& $7.710$
& $1.797$
& $17.78$
& $29.62$
& $1.25\times10^{-2}$
& $1.73\times10^{-2}$
& $0.719$
\\
100
& $1.00\times10^{-2}$
& $1.00\times10^{-3}$
& $8.665$
& $4.410$
& $7.961$
& $1.803$
& $11.75$
& $205.60$
& $1.30\times10^{-1}$
& $1.73\times10^{-1}$
& $0.752$
\\
100
& $1.00\times10^{-1}$
& $1.00\times10^{-3}$
& $8.665$
& $4.399$
& $8.248$
& $1.811$
& $11.17$
& $2028.60$
& $1.31$
& $1.73$
& $0.754$
\\

500
& $4.47\times10^{-4}$
& $4.47\times10^{-4}$
& $9.716$
& $5.144$
& $8.317$
& $1.473$
& $18.13$
& $28.71$
& $6.13\times10^{-3}$
& $8.69\times10^{-3}$
& $0.705$
\\
500
& $1.00\times10^{-2}$
& $4.47\times10^{-4}$
& $9.716$
& $5.057$
& $8.640$
& $1.479$
& $11.41$
& $426.00$
& $1.47\times10^{-1}$
& $1.94\times10^{-1}$
& $0.758$
\\
500
& $1.00\times10^{-1}$
& $4.47\times10^{-4}$
& $9.716$
& $5.052$
& $8.907$
& $1.483$
& $11.14$
& $4288.34$
& $1.48$
& $1.94$
& $0.760$
\\
\bottomrule
\end{tabular}}
\end{table*}

\begin{table*}[t]
\centering
\caption{
Comparison between the existing bound~\eqref{eqn:ucb-inst-regret-0} and our new bound~\eqref{eqn:ucb-inst-regret-new-form-solution}.
Fixed parameters:
$\sigma=10^{-2}$,
$n_L=4$,
$a=b=r=1$,
$d=2$,
$\delta=0.01$,
and
$\pi_t=\pi^2 t^2/6$.
}
\label{tab:ucb-bound-comparison-3}
\setlength{\tabcolsep}{2pt}
\renewcommand{\arraystretch}{1.15}
\scriptsize
\resizebox{\linewidth}{!}
{
\begin{tabular}{ccccccccccccc}
\toprule
$t$
&
$c_1(t)$
&
$c_2(t)$
&
$\sqrt{\beta_t}$
&
$\sqrt{\eta_t}$
&
$\sqrt{\alpha_t}$
&
$\bar s_t$
&
$\bar n_{\eta,t}$
&
$\bar n_{\alpha,t}$
&
\eqref{eqn:ucb-inst-regret-new-form-solution}
&
\eqref{eqn:ucb-inst-regret-0}
&
Ratio
\\
\midrule
20
& $2.24\times10^{-3}$
& $2.24\times10^{-3}$
& $7.879$
& $4.340$
& $7.426$
& $2.481$
& $213.64$
& $354.70$
& $2.69\times10^{-2}$
& $3.52\times10^{-2}$
& $0.764$
\\
20
& $1.00\times10^{-2}$
& $2.24\times10^{-3}$
& $7.879$
& $4.263$
& $7.586$
& $2.488$
& $150.64$
& $1160.77$
& $1.21\times10^{-1}$
& $1.58\times10^{-1}$
& $0.770$
\\
20
& $1.00\times10^{-1}$
& $2.24\times10^{-3}$
& $7.879$
& $4.238$
& $7.879$
& $2.499$
& $135.00$
& $10805.56$
& $1.21$
& $1.58$
& $0.769$
\\

100
& $1.00\times10^{-3}$
& $1.00\times10^{-3}$
& $9.021$
& $5.005$
& $8.055$
& $1.806$
& $218.00$
& $343.31$
& $1.33\times10^{-2}$
& $1.80\times10^{-2}$
& $0.737$
\\
100
& $1.00\times10^{-2}$
& $1.00\times10^{-3}$
& $9.021$
& $4.921$
& $8.293$
& $1.812$
& $141.44$
& $2327.61$
& $1.39\times10^{-1}$
& $1.80\times10^{-1}$
& $0.770$
\\
100
& $1.00\times10^{-1}$
& $1.00\times10^{-3}$
& $9.021$
& $4.911$
& $8.568$
& $1.819$
& $134.12$
& $22827.71$
& $1.39$
& $1.80$
& $0.772$
\\

500
& $4.47\times10^{-4}$
& $4.47\times10^{-4}$
& $10.035$
& $5.596$
& $8.639$
& $1.479$
& $221.15$
& $335.78$
& $6.47\times10^{-3}$
& $8.98\times10^{-3}$
& $0.721$
\\
500
& $1.00\times10^{-2}$
& $4.47\times10^{-4}$
& $10.035$
& $5.512$
& $8.948$
& $1.484$
& $137.07$
& $4884.66$
& $1.55\times10^{-1}$
& $2.01\times10^{-1}$
& $0.773$
\\
500
& $1.00\times10^{-1}$
& $4.47\times10^{-4}$
& $10.035$
& $5.508$
& $9.205$
& $1.489$
& $133.70$
& $49016.45$
& $1.55$
& $2.01$
& $0.774$
\\
\bottomrule
\end{tabular}}
\end{table*}

\begin{table*}[t]
\centering
\caption{
Comparison between the existing bound~\eqref{eqn:ucb-inst-regret-0} and our new bound~\eqref{eqn:ucb-inst-regret-new-form-solution}.
Fixed parameters:
$\sigma=10^{-2}$,
$n_L=10$,
$a=b=r=1$,
$d=6$,
$\delta=0.01$,
and
$\pi_t=\pi^2 t^2/6$.
}
\label{tab:ucb-bound-comparison-4}
\setlength{\tabcolsep}{2pt}
\renewcommand{\arraystretch}{1.15}
\scriptsize
\resizebox{\linewidth}{!}
{
\begin{tabular}{ccccccccccccc}
\toprule
$t$
&
$c_1(t)$
&
$c_2(t)$
&
$\sqrt{\beta_t}$
&
$\sqrt{\eta_t}$
&
$\sqrt{\alpha_t}$
&
$\bar s_t$
&
$\bar n_{\eta,t}$
&
$\bar n_{\alpha,t}$
&
\eqref{eqn:ucb-inst-regret-new-form-solution}
&
\eqref{eqn:ucb-inst-regret-0}
&
Ratio
\\
\midrule
20
& $2.24\times10^{-3}$
& $2.24\times10^{-3}$
& $12.825$
& $4.268$
& $11.571$
& $2.259$
& $153.83$
& $400.13$
& $3.66\times10^{-2}$
& $5.74\times10^{-2}$
& $0.638$
\\
20
& $1.00\times10^{-2}$
& $2.24\times10^{-3}$
& $12.825$
& $4.213$
& $11.685$
& $2.262$
& $120.46$
& $1431.66$
& $1.69\times10^{-1}$
& $2.57\times10^{-1}$
& $0.659$
\\
20
& $1.00\times10^{-1}$
& $2.24\times10^{-3}$
& $12.825$
& $4.196$
& $11.883$
& $2.268$
& $112.03$
& $13585.01$
& $1.70$
& $2.57$
& $0.663$
\\

100
& $1.00\times10^{-3}$
& $1.00\times10^{-3}$
& $14.476$
& $4.942$
& $12.266$
& $1.657$
& $157.28$
& $378.50$
& $1.77\times10^{-2}$
& $2.90\times10^{-2}$
& $0.611$
\\
100
& $1.00\times10^{-2}$
& $1.00\times10^{-3}$
& $14.476$
& $4.882$
& $12.435$
& $1.660$
& $115.61$
& $2852.55$
& $1.92\times10^{-1}$
& $2.90\times10^{-1}$
& $0.663$
\\
100
& $1.00\times10^{-1}$
& $1.00\times10^{-3}$
& $14.476$
& $4.875$
& $12.623$
& $1.663$
& $111.55$
& $27971.65$
& $1.93$
& $2.90$
& $0.668$
\\

500
& $4.47\times10^{-4}$
& $4.47\times10^{-4}$
& $15.957$
& $5.539$
& $12.922$
& $1.365$
& $159.91$
& $364.08$
& $8.46\times10^{-3}$
& $1.43\times10^{-2}$
& $0.593$
\\
500
& $1.00\times10^{-2}$
& $4.47\times10^{-4}$
& $15.957$
& $5.478$
& $13.143$
& $1.368$
& $113.24$
& $5923.28$
& $2.13\times10^{-1}$
& $3.19\times10^{-1}$
& $0.668$
\\
500
& $1.00\times10^{-1}$
& $4.47\times10^{-4}$
& $15.957$
& $5.475$
& $13.322$
& $1.370$
& $111.32$
& $59045.40$
& $2.14$
& $3.19$
& $0.671$
\\
\bottomrule
\end{tabular}}
\end{table*}

\section{Additional Results}
\label{se:Additional Results}

\begin{table*}[h]
\centering
\caption{Comparison of stopping criteria under the SE kernel across GP-sampled functions ($^{\dagger}$ means known hyperparameters), synthetic benchmarks, and real-world tasks. Each cell reports the mean stopping iteration (success rate / final simple regret); settings and marking follow Table~\ref{tab:results}, with the SE kernel as the surrogate model.}
\label{tab:rbf}
\setlength{\tabcolsep}{2pt}
\renewcommand{\arraystretch}{1.15}
\scriptsize
\resizebox{\linewidth}{!}
{
\begin{tabular}{lcc*{7}{c}}
\toprule
Problem & $D$ & $T$ &
\textbf{NOSTOP} &
\textbf{Oracle$_r$} &
\textbf{UCB$_{\mathrm{br}}$} &
\textbf{PRB} &
\textbf{Acq} &
\textbf{$\Delta$CB} &
\textbf{$\Delta$ES} \\
\midrule
\textbf{GP\textsuperscript{$\dagger$}\,$3{\times}10^{-3}$} & 2 & 128 &
128.0 (100\% / 0.000395) &
6.0 (100\% / 0.0164) &
70.3 (100\% / 0.000395) &
\textcolor{blue}{6.3 (98\% / 0.0172)} &
67.2 (100\% / 0.00820) &
8.6 (100\% / 0.0101) &
6.4 (98\% / 0.0183) \\
\textbf{GP\textsuperscript{$\dagger$}\,$5{\times}10^{-3}$} & 2 & 128 &
128.0 (100\% / 0.000470) &
6.3 (100\% / 0.0173) &
112.7 (100\% / 0.000470) &
\textcolor{blue}{7.2 (96\% / 0.0185)} &
67.2 (98\% / 0.0114) &
15.8 (100\% / 0.00345) &
6.4 (94\% / 0.0267) \\
\textbf{GP\textsuperscript{$\dagger$}\,$3{\times}10^{-3}$} & 6 & 512 &
512.0 (100\% / 0.00204) &
9.2 (100\% / 0.0368) &
285.0 (100\% / 0.00215) &
\textcolor{blue}{11.5 (98\% / 0.0233)} &
330.0 (68\% / 0.0695) &
25.3 (98\% / 0.00722) &
7.1 (32\% / 0.151) \\
\textbf{GP\textsuperscript{$\dagger$}\,$5{\times}10^{-3}$} & 6 & 512 &
512.0 (100\% / 0.00229) &
10.5 (100\% / 0.0334) &
440.7 (100\% / 0.00229) &
14.3 (94\% / 0.0303) &
320.4 (66\% / 0.0762) &
\textcolor{blue}{41.8 (100\% / 0.00693)} &
6.7 (14\% / 0.195) \\
\textbf{GP\,$3{\times}10^{-3}$} & 2 & 128 &
128.0 (100\% / 0.000506) &
6.5 (100\% / 0.0189) &
\textcolor{blue}{35.7 (96\% / 0.0104)} &
6.0 (88\% / 0.0343) &
64.8 (96\% / 0.0173) &
6.3 (90\% / 0.0319) &
6.6 (90\% / 0.0295) \\
\textbf{GP\,$5{\times}10^{-3}$} & 2 & 128 &
128.0 (100\% / 0.00114) &
7.1 (100\% / 0.0209) &
\textcolor{blue}{34.0 (94\% / 0.0163)} &
6.1 (84\% / 0.0440) &
64.7 (90\% / 0.0266) &
6.3 (86\% / 0.0420) &
7.2 (90\% / 0.0345) \\
\textbf{GP\,$3{\times}10^{-3}$} & 6 & 512 &
512.0 (100\% / 0.00204) &
17.4 (100\% / 0.0447) &
\textcolor{blue}{106.6 (74\% / 0.0518)} &
7.1 (22\% / 0.175) &
299.6 (64\% / 0.0823) &
8.6 (26\% / 0.165) &
6.8 (20\% / 0.170) \\
\textbf{GP\,$5{\times}10^{-3}$} & 6 & 512 &
512.0 (100\% / 0.00425) &
21.8 (100\% / 0.0458) &
\textcolor{blue}{160.4 (76\% / 0.0641)} &
7.2 (16\% / 0.190) &
297.3 (64\% / 0.0894) &
8.1 (18\% / 0.188) &
6.5 (12\% / 0.193) \\
\textbf{Branin} & 2 & 128 &
128.0 (100\% / 0.000193) &
43.2 (100\% / 0.0357) &
33.3 (42\% / 3.29) &
28.0 (40\% / 4.33) &
11.6 (2\% / 7.92) &
29.8 (44\% / 3.97) &
\textcolor{blue}{88.7 (68\% / 3.89)} \\
\textbf{Rosenbrock} & 4 & 96 &
96.0 (100\% / 0.00560) &
8.4 (100\% / 0.0379) &
78.9 (96\% / 0.0156) &
\textcolor{blue}{57.1 (96\% / 0.0223)} &
75.7 (96\% / 0.0206) &
59.9 (96\% / 0.0221) &
70.8 (98\% / 0.0117) \\
\textbf{Levy} & 4 & 96 &
96.0 (98\% / 0.0140) &
46.9 (98\% / 0.0478) &
45.4 (42\% / 0.245) &
28.3 (26\% / 0.350) &
39.8 (40\% / 0.275) &
33.9 (26\% / 0.346) &
\textcolor{blue}{49.2 (50\% / 0.235)} \\
\textbf{Hartmann} & 3 & 64 &
64.0 (96\% / 0.0297) &
45.8 (96\% / 0.0586) &
50.7 (64\% / 0.471) &
32.4 (42\% / 1.27) &
64.0 (96\% / 0.0297) &
37.3 (48\% / 1.17) &
\textcolor{blue}{50.3 (68\% / 0.365)} \\
\textbf{RFHPO} & 5 & 128 &
128.0 (94\% / 0.0181) &
22.8 (94\% / 0.0251) &
\textcolor{blue}{60.4 (72\% / 0.0247)} &
12.1 (50\% / 0.0319) &
6.5 (32\% / 0.0367) &
13.6 (48\% / 0.0326) &
6.2 (32\% / 0.0374) \\
\textbf{FMNIST-SVM} & 2 & 64 &
64.0 (90\% / 0.156) &
17.0 (90\% / 0.160) &
\textcolor{blue}{27.3 (76\% / 0.158)} &
9.3 (56\% / 0.163) &
6.1 (48\% / 0.165) &
9.8 (58\% / 0.163) &
6.1 (48\% / 0.165) \\
\bottomrule
\end{tabular}
}
\end{table*}

As mentioned in Section~\ref{se:Experiments}, Table~\ref{tab:rbf} presents experimental results under the SE kernel.
As in Table~\ref{tab:results}, the marking rule highlights the practical method that stops in the fewest iterations among those reaching success rate $\ge 1-\delta$, and otherwise the one with the highest success rate; see the table caption for the full specification.
The overall trends are consistent with those observed under the Mat\'ern kernel. UCB$_{\mathrm{br}}$ continues to achieve a favorable balance between evaluation savings and success rate across most problem settings.
On the synthetic benchmarks, UCB$_{\mathrm{br}}$ achieves $96\%$ success rate on Rosenbrock, above the target $1-\delta=95\%$. For the HPO tasks, UCB$_{\mathrm{br}}$ achieves $72\%$ and $76\%$ success while reducing evaluations by $53\%$ and $57\%$ compared to NOSTOP. On the GP-learned $6$D benchmarks, UCB$_{\mathrm{br}}$ achieves $74\%$ and $76\%$ success---well above PRB ($16$--$22\%$), Acq ($64\%$), and $\Delta$CB ($18$--$26\%$)---demonstrating robustness in the harder, hyperparameter-learning regime.
Overall, UCB$_{\mathrm{br}}$ generally outperforms or matches the baseline stopping methods across the majority of test cases, confirming the robustness of the proposed criterion across different kernel choices.

\subsection{Ablation Study}\label{se:ablation-tables}

\subsubsection{Bayesian Variant Ablation}

Table~\ref{tab:ablation-B} reports the sensitivity to the constants $a,b$ of
Assumption~\ref{assp:bayesian} at the default $\delta=0.05$. Table~\ref{tab:ablation-delta} reports
$\delta$-sensitivity at the default $a=b=1$, and Table~\ref{tab:ablation-eps} reports $\epsilon$-sensitivity at both defaults.
The default configuration $(a=b=1,\delta=0.05,\epsilon=0.1)$ is the same column
in all three tables and is highlighted in bold.

\begin{table*}[h]
\centering
\caption{Sensitivity of UCB$_{\mathrm{br}}$ to the constants $a$ and $b$ of
Assumption~\ref{assp:bayesian}, which enter $\beta_t$,
at fixed $\delta=0.05$, under the Mat\'ern-$5/2$ kernel. Each arm changes one constant
by a decade and keeps the other at $1$.
Cells follow the format of Table~\ref{tab:results}: mean stopping
iteration (success rate $/$ final simple regret). The default
$(a{=}b{=}1)$ column is in bold and repeats the UCB$_{\mathrm{br}}$ column of
Table~\ref{tab:results}. Each cell averages over $50$ independent random seeds.}
\label{tab:ablation-B}
\setlength{\tabcolsep}{4pt}
\renewcommand{\arraystretch}{1.15}
\small
\resizebox{\linewidth}{!}{
\begin{tabular}{lcc*{4}{c}}
\toprule
Problem & $D$ & $T$ &
$\boldsymbol{a=b=1\ \text{(default)}}$ & $a=10$ & $b=0.1$ & $b=10$ \\
\midrule
\textbf{Branin} & 2 & 128 &
\textbf{87.0 (88\% / 0.313)} &
86.5 (84\% / 1.01) &
77.1 (84\% / 1.57) &
91.0 (86\% / 0.905) \\
\textbf{Rosenbrock} & 4 & 96 &
\textbf{80.4 (100\% / 0.00868)} &
81.2 (98\% / 0.0119) &
84.7 (98\% / 0.0101) &
78.7 (98\% / 0.0144) \\
\textbf{Hartmann} & 3 & 64 &
\textbf{53.2 (70\% / 0.388)} &
49.9 (62\% / 0.417) &
56.5 (86\% / 0.210) &
47.6 (62\% / 0.523) \\
\textbf{GP\,$3{\times}10^{-3}$} & 2 & 128 &
\textbf{29.4 (94\% / 0.0210)} &
35.5 (92\% / 0.0203) &
26.4 (92\% / 0.0185) &
35.5 (92\% / 0.0197) \\
\textbf{GP\,$5{\times}10^{-3}$} & 2 & 128 &
\textbf{38.2 (92\% / 0.0232)} &
31.9 (92\% / 0.0204) &
29.4 (90\% / 0.0212) &
45.4 (92\% / 0.0190) \\
\textbf{GP\,$3{\times}10^{-3}$} & 6 & 512 &
\textbf{123.9 (72\% / 0.0511)} &
158.8 (76\% / 0.0456) &
150.7 (78\% / 0.0409) &
168.8 (76\% / 0.0568) \\
\textbf{GP\,$5{\times}10^{-3}$} & 6 & 512 &
\textbf{79.9 (60\% / 0.0865)} &
105.9 (58\% / 0.0952) &
137.5 (72\% / 0.0627) &
191.5 (72\% / 0.0581) \\
\bottomrule
\end{tabular}}
\end{table*}

\begin{table*}[h]
\centering
\caption{Sensitivity of UCB$_{\mathrm{br}}$ to the failure
probability $\delta$ at fixed $a=b=1$, under the Mat\'ern-$5/2$
kernel. Cells follow the format of Table~\ref{tab:results}.
The default $(\delta{=}0.05)$ column is in bold and repeats the UCB$_{\mathrm{br}}$ column of
Table~\ref{tab:results}. Each cell averages over
$50$ independent random seeds.}
\label{tab:ablation-delta}
\setlength{\tabcolsep}{4pt}
\renewcommand{\arraystretch}{1.15}
\small
\resizebox{\linewidth}{!}{
\begin{tabular}{lcc*{3}{c}}
\toprule
Problem & $D$ & $T$ &
$\delta=0.01$ & $\boldsymbol{\delta=0.05\ \text{(default)}}$ & $\delta=0.2$ \\
\midrule
\textbf{Branin} & 2 & 128 &
84.1 (78\% / 1.67) &
\textbf{87.0 (88\% / 0.313)} &
80.7 (80\% / 2.10) \\
\textbf{Rosenbrock} & 4 & 96 &
81.1 (100\% / 0.0105) &
\textbf{80.4 (100\% / 0.00868)} &
79.9 (98\% / 0.0138) \\
\textbf{Hartmann} & 3 & 64 &
50.0 (62\% / 0.433) &
\textbf{53.2 (70\% / 0.388)} &
53.4 (72\% / 0.301) \\
\textbf{GP\,$3{\times}10^{-3}$} & 2 & 128 &
35.7 (92\% / 0.0191) &
\textbf{29.4 (94\% / 0.0210)} &
31.4 (92\% / 0.0191) \\
\textbf{GP\,$5{\times}10^{-3}$} & 2 & 128 &
43.3 (90\% / 0.0225) &
\textbf{38.2 (92\% / 0.0232)} &
30.6 (90\% / 0.0230) \\
\textbf{GP\,$3{\times}10^{-3}$} & 6 & 512 &
150.1 (78\% / 0.0451) &
\textbf{123.9 (72\% / 0.0511)} &
160.4 (74\% / 0.0519) \\
\textbf{GP\,$5{\times}10^{-3}$} & 6 & 512 &
160.8 (68\% / 0.0649) &
\textbf{79.9 (60\% / 0.0865)} &
167.4 (66\% / 0.0692) \\
\bottomrule
\end{tabular}}
\end{table*}

The certificate has no calibration constant beyond $\delta$ and the constants $a,b$ of Assumption~\ref{assp:bayesian}, which enter $\beta_t$ and the discretization size $|D_t|$ only through $\log\log(\cdot)$ and $\log b$, respectively. Perturbing either by a decade keeps the success rate within $2$ points of the default on the GP $d{=}2$ and Rosenbrock rows, and raises it on the learned GP $d{=}6$ rows under both $b$ perturbations ($72\%\to 78\%$ and $76\%$; $60\%\to 72\%$) and on Hartmann at $b{=}0.1$ ($70\%\to 86\%$) (Table~\ref{tab:ablation-B}).

The effect of $\delta$ on the success rate is small: $\delta$ enters $\beta_t$ only logarithmically, and varying it by a factor of $20$ keeps the success rate of UCB$_{\mathrm{br}}$ within $2$ points of the default on the GP $d{=}2$ and Rosenbrock rows and raises it on the GP $d{=}6$ rows (Table~\ref{tab:ablation-delta}). We restrict the sweep to Mat\'ern-$5/2$ on the GP and analytic families: the well-specified GP$^\dagger$ rows of Table~\ref{tab:results} reach $100\%$ success and admit no informative sweep at $\epsilon{=}0.1$, and SE/RBF results in Table~\ref{tab:rbf} follow the same qualitative pattern.

Table~\ref{tab:ablation-eps} varies the target accuracy $\epsilon$. On every row, a larger $\epsilon$ stops no later and succeeds at least as often: at $\epsilon=0.2$, UCB$_{\mathrm{br}}$ reaches $100\%$ success on both GP $d{=}2$ rows using $9$--$11\%$ of the budget, and $96\%$ and $88\%$ on the GP $d{=}6$ rows using $10\%$ of the budget.

\begin{table*}[h]
\centering
\caption{Sensitivity of UCB$_{\mathrm{br}}$ to the target accuracy $\epsilon$ at fixed
$\delta=0.05$ and $a=b=1$, under the Mat\'ern-$5/2$ kernel. Cells follow the format of Table~\ref{tab:results}, with
the success rate measured against the column's $\epsilon$. The default $(\epsilon{=}0.1)$
column is in bold and repeats the UCB$_{\mathrm{br}}$ column of Table~\ref{tab:results}.
Each cell averages over $50$ independent
random seeds.}
\label{tab:ablation-eps}
\setlength{\tabcolsep}{4pt}
\renewcommand{\arraystretch}{1.15}
\small
\resizebox{\linewidth}{!}{
\begin{tabular}{lcc*{5}{c}}
\toprule
Problem & $D$ & $T$ &
$\epsilon=0.05$ & $\epsilon=0.075$ & $\boldsymbol{\epsilon=0.1\ \text{(default)}}$ & $\epsilon=0.15$ & $\epsilon=0.2$ \\
\midrule
\textbf{Branin} & 2 & 128 &
105.6 (84\% / 0.311) &
93.2 (84\% / 0.312) &
\textbf{87.0 (88\% / 0.313)} &
80.9 (88\% / 0.314) &
79.3 (88\% / 0.315) \\
\textbf{Rosenbrock} & 4 & 96 &
80.4 (98\% / 0.00868) &
80.4 (100\% / 0.00868) &
\textbf{80.4 (100\% / 0.00868)} &
80.4 (100\% / 0.00868) &
80.4 (100\% / 0.00868) \\
\textbf{Hartmann} & 3 & 64 &
53.2 (70\% / 0.388) &
53.2 (70\% / 0.388) &
\textbf{53.2 (70\% / 0.388)} &
53.2 (72\% / 0.388) &
53.0 (76\% / 0.388) \\
\textbf{GP\,$3{\times}10^{-3}$} & 2 & 128 &
49.8 (90\% / 0.0210) &
42.0 (92\% / 0.0210) &
\textbf{29.4 (94\% / 0.0210)} &
15.4 (98\% / 0.0219) &
12.1 (100\% / 0.0233) \\
\textbf{GP\,$5{\times}10^{-3}$} & 2 & 128 &
47.0 (88\% / 0.0232) &
44.6 (90\% / 0.0232) &
\textbf{38.2 (92\% / 0.0232)} &
22.6 (98\% / 0.0245) &
14.6 (100\% / 0.0261) \\
\textbf{GP\,$3{\times}10^{-3}$} & 6 & 512 &
149.1 (66\% / 0.0511) &
140.9 (72\% / 0.0511) &
\textbf{123.9 (72\% / 0.0511)} &
76.6 (90\% / 0.0511) &
51.6 (96\% / 0.0511) \\
\textbf{GP\,$5{\times}10^{-3}$} & 6 & 512 &
84.6 (44\% / 0.0865) &
84.1 (52\% / 0.0865) &
\textbf{79.9 (60\% / 0.0865)} &
67.4 (74\% / 0.0865) &
51.0 (88\% / 0.0865) \\
\bottomrule
\end{tabular}}
\end{table*}

\subsubsection{Tightened versus Existing Bound}\label{se:ablation-bound}

Table~\ref{tab:ablation-bound} isolates the effect of our new bound~\eqref{eqn:ucb-inst-regret-new-form-solution}. The criterion is evaluated with it and with the existing bound~\eqref{eqn:ucb-inst-regret-0} using the same $\beta_t$ along the same BO trajectories, so the two differ only in the bound. Across all $14$ problems of Table~\ref{tab:results}, the new bound is never wider than the existing one at any of the $151{,}700$ recorded iterations, and its median width is $36\%$ smaller. On the well-specified GP$^{\dagger}$ problems, the existing bound triggers on $18$ of the $50$ seeds at $d{=}2$, $\sigma{=}3{\times}10^{-3}$ and on none of the seeds at $d{=}2$, $\sigma{=}5{\times}10^{-3}$ and $d{=}6$, $\sigma{=}3{\times}10^{-3}$, so the criterion would run the full budget there; the new bound triggers on $48$, $24$, and $12$ seeds with the same final regret.

\begin{table*}[h]
\centering
\caption{UCB$_{\mathrm{br}}$ with our new bound~\eqref{eqn:ucb-inst-regret-new-form-solution}
and with the existing bound~\eqref{eqn:ucb-inst-regret-0}, both evaluated with the same
$\beta_t$ along the same BO trajectories, on the well-specified GP$^{\dagger}$
problems under the Mat\'ern-$5/2$ kernel. Cells follow the format of
Table~\ref{tab:results}; the number of the $50$ seeds on which the certificate triggered is
given in brackets.}
\label{tab:ablation-bound}
\setlength{\tabcolsep}{4pt}
\renewcommand{\arraystretch}{1.15}
\small
\begin{tabular}{lcccc}
\toprule
Problem & $D$ & $T$ & New bound~\eqref{eqn:ucb-inst-regret-new-form-solution} & Existing bound~\eqref{eqn:ucb-inst-regret-0} \\
\midrule
\textbf{GP$^{\dagger}$\,$3{\times}10^{-3}$} & 2 & 128 & 76.6 (100\% / 0.000347) [48] & 122.1 (100\% / 0.000347) [18] \\
\textbf{GP$^{\dagger}$\,$5{\times}10^{-3}$} & 2 & 128 & 115.2 (100\% / 0.000899) [24] & 128.0 (100\% / 0.000899) [0] \\
\textbf{GP$^{\dagger}$\,$3{\times}10^{-3}$} & 6 & 512 & 483.3 (100\% / 0.00202) [12] & 512.0 (100\% / 0.00202) [0] \\
\textbf{GP$^{\dagger}$\,$5{\times}10^{-3}$} & 6 & 512 & 509.9 (100\% / 0.00297) [1] & 512.0 (100\% / 0.00297) [0] \\
\bottomrule
\end{tabular}
\end{table*}

\subsubsection{Kernel-Smoothness Misspecification}\label{se:ablation-nu}

Table~\ref{tab:ablation-nu} misspecifies the kernel smoothness of the surrogate. UCB$_{\mathrm{br}}$ is $\epsilon$-optimal in all $200$ runs, matching NOSTOP while saving $22$--$56\%$ of the evaluations, and its final regret is $3$--$10$ times smaller than that of PRB, which fails on $2$--$3$ of the $50$ seeds in every setting.

\begin{table*}[h]
\centering
\caption{Robustness of UCB$_{\mathrm{br}}$ to kernel-smoothness misspecification. The
objective is a Mat\'ern-$5/4$ draw while the surrogate assumes $\nu=1.5$ or $\nu=2.5$; in
GP$^{\dagger}$ the true lengthscale, outputscale, and noise are supplied, so the smoothness is
the only misspecified quantity. Cells follow the format of Table~\ref{tab:results}; $50$
seeds each.}
\label{tab:ablation-nu}
\setlength{\tabcolsep}{4pt}
\renewcommand{\arraystretch}{1.15}
\small
\resizebox{\linewidth}{!}{
\begin{tabular}{clcc*{3}{c}}
\toprule
Surrogate $\nu$ & Problem & $D$ & $T$ & NOSTOP & UCB$_{\mathrm{br}}$ & PRB \\
\midrule
$1.5$ & \textbf{GP$^{\dagger}$\,$3{\times}10^{-3}$} & 2 & 128 &
128.0 (100\% / 0.00255) &
99.6 (100\% / 0.00255) &
6.6 (96\% / 0.0269) \\
$1.5$ & \textbf{GP\,$3{\times}10^{-3}$} & 2 & 128 &
128.0 (100\% / 0.00403) &
61.2 (100\% / 0.0115) &
6.4 (94\% / 0.0398) \\
$2.5$ & \textbf{GP$^{\dagger}$\,$3{\times}10^{-3}$} & 2 & 128 &
128.0 (100\% / 0.00559) &
66.2 (100\% / 0.00645) &
6.3 (96\% / 0.0280) \\
$2.5$ & \textbf{GP\,$3{\times}10^{-3}$} & 2 & 128 &
128.0 (100\% / 0.00460) &
56.0 (100\% / 0.0100) &
6.2 (96\% / 0.0335) \\
\bottomrule
\end{tabular}}
\end{table*}

\end{document}